\documentclass[accepted]{uai2026} 
                        
\usepackage[american]{babel}
\usepackage{microtype}
\usepackage{graphicx}
\usepackage{wrapfig}
\usepackage{subcaption}
\usepackage{booktabs} 

\usepackage{natbib} 
\usepackage{mathtools} 
\usepackage{booktabs} 
\usepackage{tikz} 
\usepackage[ruled,vlined]{algorithm2e}

\usepackage{amsmath}
\usepackage{amssymb}
\usepackage{mathtools}
\usepackage{amsthm}
\usepackage{cuted}
\usepackage{hyperref}
\hypersetup{hidelinks}
\usepackage[nameinlink,capitalise]{cleveref}

\definecolor{linkblue}{RGB}{52,120,184}

\theoremstyle{plain}
\newtheorem{theorem}{Theorem}[section]
\newtheorem{proposition}{Proposition}
\newtheorem{lemma}[theorem]{Lemma}

\theoremstyle{definition}

\theoremstyle{remark}

\title{Robust Bayesian Decision Making under Adversarial Uncertainty}

\author[1]{\href{mailto:<haripriya.harikumar@manchester.ac.uk>?Subject=Your UAI 2026 paper}{Haripriya Harikumar}{}}
\author[2,3]{Sammie Katt}
\author[1]{Yasir Zubayr Barlas}
\author[1,2,3]{Samuel Kaski}

\affil[1]{%
    Department of Computer Science\\ The University of Manchester, UK
}
\affil[2]{%
    ELLIS Institute Finland
}
\affil[3]{%
    Department of Computer Science\\
    Aalto University, Espoo\\
    Finland
  }
  
\begin{document}
\maketitle

\begin{abstract}
Scientific experiments are often designed to maximize information gain, yet in many applications the primary objective is to support reliable downstream decision-making. Existing decision-aware experimental design and active learning methods typically assume well-specified outcome models and implicitly rely on the stability of the optimal decision under real-world perturbations. In practice, however, experimental outcomes are frequently influenced by hidden or weakly modeled effects, which can substantially alter decision optimality and lead to misleading conclusions. We study sequential adversarially robust decision-aware experimental design, where data acquisition has to take into account information gain against plausible worst-case unexpected effects, modeled here as variation in adversarial variables. Building on Bayesian decision theory, we formalize an adversarially robust optimal decision under this setting and derive a principled Bayesian experimental design criterion. The criterion explicitly targets decision stability rather than nominal optimality. Experiments on synthetic and real-world scientific datasets show that conventional decision-aware design can converge rapidly to high confidence yet fragile decisions, while our robustness-aware approach yields decisions that are significantly more stable and reliable under adversarial variation.
\end{abstract}

\section{Introduction}
Science advances through controlled experimentation \citep{fisher1935design,cheng2005bayesian,melendez2021designing}: experiments are designed by manipulating a set of controllable variables, observing outcomes, and drawing conclusions about an underlying process. Classical experimental design \citep{chaloner1995bayesian,ryan2007modern} methods aim to test hypotheses or maximize information gain. In practice, however, experimentation is costly, resources are limited, and not all relevant variables can be identified, measured, or controlled. As a result, experimental outcomes are often influenced by weakly modeled adversarial factors that distort outcomes and degrade the reliability of downstream decisions \citep{corbett2017algorithmic,lacoste2011approximate}.
\begin{figure}
    \centering
    \includegraphics[width=1.0\linewidth]{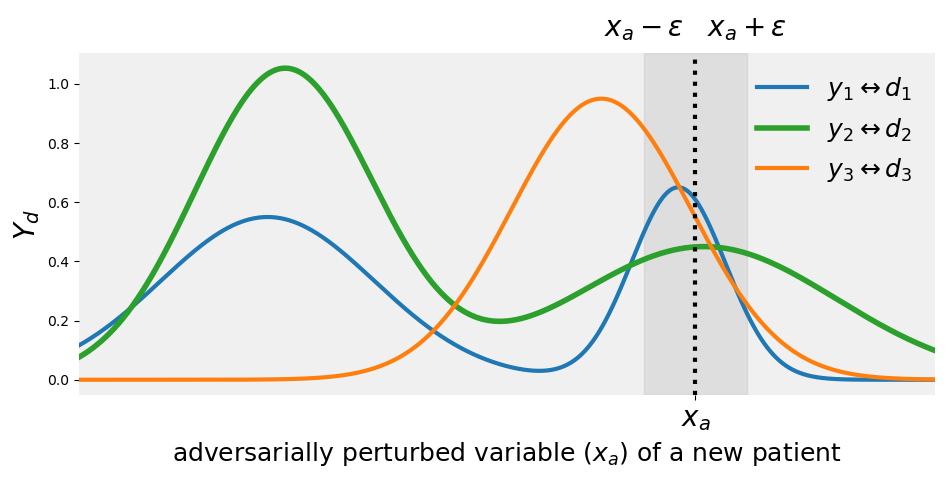}
    \caption{Three treatment plans (in green, orange, and blue curves) and their corresponding outcomes on the y-axis. A clinician must select a treatment for a new patient (indicated by the black dashed line). The x-axis denotes adversarially perturbed variable associated with patients. Although the blue treatment appears optimal (highest outcome) for the new patient, the smoother nominal utility (green) leads to decisions that remain stable across the perturbation region (shaded grey area denoted as $x_a-\epsilon$ and $x_a+\epsilon$ ), motivating adversarially robustness-aware decision learning.}
    \label{fig:intuition}
\end{figure}
This challenge is particularly significant in high-dimensional settings, where some potential adversarial variables may systematically affect outcomes in ways that are difficult to model probabilistically. Such factors may be external, latent, partially observable, or implicitly coupled to the experimental process \citep{grunwald2017inconsistency,rainforth2024modern}, and their influence often becomes apparent only after decisions are made. Consequently, experimental designs that are optimal for parameter inference or predictive accuracy can lead to fragile or misleading decision-making when these influences are present. This raises a fundamental question: how can experiments be designed to support decisions that remain reliable under variations in adversarial variables?

The same abstract problem occurs in decision-making in personalized medicine, for instance, when a new patient arrives, the doctor must decide on the most appropriate treatment strategy. However, the doctor does not know beforehand, how a particular patient responds to a particular treatment. An interesting example of experimental design to gather indirect evidence is from active learning for personalized medicine \citep{bica2021real,sundin2019active}. The algorithm serves as an information acquisition policy that assists the doctor in determining which additional patient data from an existing unlabeled patient set (outcome of the treatment is not known) would be the most informative to improve treatment decisions. Once the doctor or medical system retrieves the corresponding medical outcome (e.g. contact the physician who are treating them), the algorithm updates its predictive model of the response to treatment. The final decision lies with the physician, who integrates algorithmic recommendations with clinical expertise, ethical considerations, and patient specific constraints. However, while such decisions based on active learning may be optimal under nominal conditions, they can lack robustness to unexpected variations during treatment. This becomes particularly important in the presence of adversarial variables - factors that can significantly influence outcomes if not properly modeled. Therefore, a truly robust and reliable treatment decision must explicitly consider variations against the plausible worst-case unexpected effects in these adversarial variables to ensure stability in medical outcomes (shown in \cref{fig:intuition}). 

Motivated by this, we view scientific experimentation as a process that must account for uncertainty not only in model parameters, but also in the stability of decisions under adversarial variable variation. Rather than focusing solely on identifying parameters or predictions that are optimal under nominal assumptions, we argue that experimental design should explicitly target robust decision quality, that is, decisions whose optimality persists under plausible worst-case effects of unmodeled factors \citep{krause2007selecting}.

Building on Bayesian decision theory \cite{berger2013statistical}, we develop an adversarially robust, decision-aware experimental design framework that explicitly optimizes data acquisition for downstream decision reliability. Our approach differs from standard Bayesian optimal experimental design by shifting the focus from parameter inference to adversarially robust decision utility, thereby aligning experimental design with the ultimate goal of making stable and reliable decisions in complex, real-world systems.

\section{Background}
\subsection{Bayesian Optimal Experimental Design}
Bayesian optimal experimental design (BOED) \citep{lindley1956measure,rainforth2024modern} is a framework for selecting optimal designs $\xi \in \Xi$ to acquire data $y$ to obtain information about a parameter of interest $\theta$. This is done by selecting designs that maximize the expected reduction in entropy from the prior $p(\theta)$ to the posterior $p(\theta\mid \xi, y)$. In other words, we maximize the expected information gain (EIG)
\begin{equation} \label{eq:1}
    \text{EIG}(\xi) = \mathbb{E}_{p(y \mid \xi)} [H[p(\theta)]-H[p(\theta\mid \xi, y)]], 
\end{equation}
where $H[\cdot]$ denotes the Shannon entropy \citep{shannon1948mathematical}, $H[p(\cdot)]=-\mathop{\mathbb{E}}_{p(\cdot)}[\log p(\cdot)]$. The optimal design is $\xi^*=\text{argmax}_{\xi \in \Xi} \text{EIG}(\xi)$. As proposed by \cite{houlsby2011bayesian}, \cref{eq:1} can re-written in the form of entropies on the outcome space $y$ as follows
\begin{equation} \label{eq:1a}
    \text{EIG}(\xi) = \mathbb{E}_{p(\theta \mid D)} [H[p(y\mid \xi)]-H[p(y\mid \xi, \theta)]], 
\end{equation}
where $D=\{(\xi_i,y_i)\}_{1:N}$ contains the $N$ design-observation pairs obtained thus far.
\subsection{Bayesian Decision Theory}
Bayesian decision theory \citep{berger2013statistical} is effective for decision-making under uncertainty. It has a decision utility function denoted as $U(\theta,d)$, where $\theta$ is the state of the system when making a decision $d \in \mathcal{S}$. This utility function quantifies the usefulness of taking the decision $d \in \mathcal{S}$ when the system is in state $\theta$. The optimal decision is obtained by maximizing the expected utility
\begin{equation}
    d^* = \arg\max_{d \in \mathcal{S}} \mathbb{E}_{p(\theta \mid D)} [U(\theta,d)].
\end{equation}
In many applications, decisions influence downstream outcomes $y$, making it more natural to define utility over $y$ than $\theta$ \citep{kusmierczyk2019variational,vadera2021post}. In such cases, we write the utility as $U(y_\Xi,d)$ defined over all possible predicted outcomes $y_\Xi$ when making a decision $d \in \mathcal{S}$. As a result, the optimal decision can be expressed in terms of the predicted outcomes by
\begin{equation}
    d^* = \arg\max_{d \in \mathcal{S}} \mathbb{E}_{p(y_\Xi\mid D)} [U(y_\Xi,d)],
\end{equation}
where $p(y_\Xi\mid D) = \{p(y\mid \xi, D)\}_{\xi \in \Xi}$, a joint predictive distribution of outcomes over all possible designs given the current information in $D$ and $p(y\mid \xi, D)= \mathop{\mathbb{E}}_{p(\theta \mid D)}[ p(y\mid \xi, \theta,D)]$ is the predictive distribution \citep{huang2024amortized}.
\subsection{Bayesian Active Learning}
In a resource-constrained setup, active learning \citep{mccallum1998employing,dasgupta2004analysis,golovin2011adaptive,settles2012active} looks for samples that help the decision making process with minimal sampling. These samples are an unlabeled pool of data and the decision-maker would observe the label that would eventually help them make an optimal decision. Consider a training dataset $D=\{(\bold{\xi}_i,d_i,y_i)\}_{1:N}$ and an unlabeled query pool $D_{q}=\{(\bold{\xi}_i,d_i)\}_{1:N}$. Starting from the EIG in \cref{eq:1a}, one queries from $D_{q}$ by
\begin{equation}\label{eq:2}
\begin{aligned}
   \text{PEIG}(\xi^*,d^*) = \arg\max_{(\xi_j,d_j)\in D_q}(H[p(y_j \mid (\xi_j,d_j),D)]\\
-  \mathbb{E}_{p(\theta \mid D)}[  
    H[p(y_j \mid (\xi_j,d_j),\theta)]]).
\end{aligned}
\end{equation}
\cref{eq:2} can be further refined by considering a new design $\tilde{\xi}$, for which the decision-maker must make a decision by querying the unlabeled query pool, as follows,
\begin{equation}\label{eq:3}
\begin{aligned}
   \text{TEIG}(\xi^*,d^*) = \arg\max_{(\xi_j,d_j)\in D_q}(H[p(\tilde{y} \mid \ \tilde{\xi}, D)]\\
-  \mathbb{E}_{p(y_j \mid (\xi_j, d_j),D)} [ 
    H[p(\tilde{y} \mid \tilde{\xi},D\cup \{(\xi_j,d_j,y_j)\})]]).
\end{aligned}
\end{equation}
Here, $\xi_j$ is from the query pool $D_q$, and $y_j$ is computed from the predictive model trained with $D$.

\citet{filstroff2024targeted} characterize the optimal decision into a quantity of interest, which is denoted as $\mathcal{D}_{\text{best}}(\tilde \xi)$, defined as a discrete random variable with probability mass function $\{ \pi_d \}_{d=1}^{|\mathcal{S}|}$, where $\pi_d$ is the posterior probability that $d$ is the optimal decision for a new design $\tilde \xi$. We define a utility function $U_d(\xi,d) \rightarrow \mathbb{R}$, or simply $U_d(\xi) \rightarrow \mathbb{R}$, which quantifies the quality of taking a decision $d$ when the design is $\xi$. So, $\pi_d$ for $\tilde \xi$ can be defined as follows,
\begin{equation}\label{pi_deig}
    \pi_d = \mathbb{P}(d=\arg\max_{d' \in \mathcal{S}} U_{d'}(\tilde\xi_j,d')).
\end{equation}
So, for a new design $\tilde \xi$, \cref{eq:3} can be refined by including the new quantity of interest, i.e., the posterior distribution of decision uncertainty $\mathcal{D}_{\text{best}}(\tilde \xi)$ is taken into account,
\begin{equation}\label{eq:4}
\begin{aligned}
   \text{DEIG}(\xi^*,d^*) = \arg\max_{(\xi_j,d_j)\in D_q}(H[p(\mathcal{D}_{\text{best}}(\tilde\xi) \mid D)]\\
-  \mathbb{E}_{y_j \sim p(y_j \mid (\xi_j,d_j), D)}[  
    H[p(\mathcal{D}_{\text{best}}(\tilde\xi) \mid D \cup \{(\xi_j,d_j, y_j)\})]]).
\end{aligned}
\end{equation}

\section{Adversarially Robust Bayesian Decision Making}\label{ar-deig}
While downstream decision-making is the ultimate objective, it is crucial to account for adversarial variables that may undermine these decisions if we optimize solely for immediate performance. Such variables can have long-term detrimental effects, causing decisions that appear optimal in the short-term to fail under future conditions. Therefore, it is essential to model and anticipate changes in these variables, aiming for decisions that are both robust and optimal.

For simplicity, we assume that the decision-maker is able to identify these variables, for example a clinician defining realistic variation in cholesterol readings. This is practically relevant because domain experts are part of the process. So our proposed method remains agnostic to how these variables are identified and can be combined with sensitivity analysis, domain-driven feature selection, or data-driven instability detection. 

Consider nominal design variables as $\xi_t$ and adversarial design variables as $\xi_a$ (both vectors) and denoted as $\bold{\xi} = (\bold{\xi}_t,\bold{\xi}_a)$. The downstream decision space is denoted as $\mathcal{S}$. The outcomes observed after taking a decision $d \in \mathcal{S}$ is represented as $y_i \in \mathcal{Y}$. The dataset $D=\{((\bold{\xi}_t,\bold{\xi}_a)_i,d_i,y_i)\}_{1:N}$ contains the design variables, adversarial design variables, the decision, and the outcome of the decision. 
Our main objective is to query a design from the query pool that will help us make an adversarially robust decision for a new design. For each decision $d \in \mathcal{S}$, we define the latent utility function 
$U_d(\xi_t,\xi_a) \rightarrow \mathbb{R}$, which quantifies the quality of taking a decision $d$ when the design is $(\xi_t,\xi_a)$. Throughout this work, we set $U_d=y_d$, which is the outcome we estimate from the predictive model for a decision $d$.
For adversarial robustness, in this paper we consider the case where we want to build robustness to adversarial perturbations for the adversarial variable $\xi_a$ up to a user-specified level epsilon ($\epsilon$) as follows, \[
\mathcal{A}_{\epsilon}(\xi_a) = \{\xi_a' : \|\xi_a' - \xi_a\| \le \epsilon\},
\] where the value of $\epsilon$ controls the scope of the adversarial perturbation.
This has a natural game-theoretic interpretation. 
Specifically, the decision-making problem can be viewed as a two-player 
Stackelberg game \citep{von1947theory,grunwald2004game}  in which the decision-maker first selects a decision 
$d \in \mathcal{S}$, and an adversary subsequently chooses a perturbation 
$\xi_a' \in \mathcal{A}_{\epsilon}(\xi_a)$ to minimize the resulting utility. In scientific settings we do not necessarily assume that the perturbation is adversarial, but this is a technical way of deriving worst-case guarantees. Under this perspective, the objective is to select the decision that 
maximizes the worst-case latent utility,
\[
d^*(\xi_t,\xi_a;\epsilon)
=
\arg\max_{d \in \mathcal{S}}
\min_{\xi_a' \in \mathcal{A}_{\epsilon}(\xi_a)}
U_d(\xi_t,\xi_a').
\]

Our formulation is closely related to the targeted active learning
framework of \citet{filstroff2024targeted}, which
connects utility modeling with downstream decision-making. We redefine the target of information as a robust decision random variable. Adapting this perspective to an adversarial setting, we define the
robust value of each decision as the worst-case latent utility over
the admissible perturbation set:
\begin{equation}\label{utility}
V_d(\xi_t,\xi_a;\epsilon)=
\min_{\xi_a' \in \mathcal{A}_{\epsilon}(\xi_a)}
U_d(\xi_t,\xi_a').
\end{equation}

The robust-optimal decision is therefore given by,
\begin{equation}
d^*(\xi_t,\xi_a;\epsilon)
=
\arg\max_{d \in \mathcal{S}}
V_d(\xi_t,\xi_a;\epsilon).  
\end{equation}
\begin{proposition}\label{prop1}
Assume that the latent utility $U_d(\xi_t,\xi_a')$ is well-defined for $\xi_a' \in \mathcal{A}_{\epsilon}(\xi_a)$. Then, for every decision $d \in \mathcal{S}$, $V_d(\xi_t,\xi_a;\epsilon)\leq U_d(\xi_t,\xi_a')$, and consequently, $\max_{d \in \mathcal{S}} V_d(\xi_t,\xi_a;\epsilon) \leq \max_{d \in \mathcal{S}}U_d(\xi_t,\xi_a')$.
\end{proposition}
\begin{proof}
Since $\xi_a' \in \mathcal{A}(\xi_a)$, the nominal utility $U_d(\xi_t,\xi_a')$ is one of the candidates in the minimization defining $V_d(\xi_t,\xi_a;\epsilon)$. Therefore $V_d(\xi_t,\xi_a;\epsilon) \leq U_d(\xi_t,\xi'_a)$. Taking the maximum over $d \in \mathcal{S}$ on both sides gives
$\max_{d \in \mathcal{S}} V_d(\xi_t,\xi_a;\epsilon) \leq \max_{d \in \mathcal{S}} U_d(\xi_t,\xi'_a)$.
\end{proof}
 \cref{prop1} formalizes the cost of adversarial robustness in decision making, meaning optimizing decisions for worst-case perturbations cannot give higher utility than the nominal optimum.

\begin{proposition}\label{prop2}
Let $0 \le \epsilon_1 \le \epsilon_2$, and let the corresponding adversarial perturbation sets be $\mathcal{A}_{\epsilon_1}(\xi_a)\subseteq \mathcal{A}_{\epsilon_2}(\xi_a)$. Then, for every decision $d \in \mathcal{S}$, the robust utility is non-increasing with respect to the adversarial radius $\epsilon$.
\end{proposition}
\begin{proof}
Since $\epsilon_1 \le \epsilon_2$, we have 
$\mathcal{A}_{\epsilon_1}(\xi_a) \subseteq \mathcal{A}_{\epsilon_2}(\xi_a)$. Hence,
\[
\begin{aligned}
V_d(\xi_t,\xi_a;{\epsilon_2})
&= \min_{\xi_a' \in \mathcal{A}_{\epsilon_2}(\xi_a)}
U_d(\xi_t,\xi_a') \\
&\le
\min_{\xi_a' \in \mathcal{A}_{\epsilon_1}(\xi_a)}
U_d(\xi_t,\xi_a') &= V_d(\xi_t,\xi_a;{\epsilon_1}).
\end{aligned}
\]
Taking the maximum over $d \in \mathcal{S}$ preserves the inequality.
\end{proof}
\cref{prop2} formalizes that increasing the adversarial budget can only make decisions more conservative: as the perturbation set expands, the worst-case utility cannot improve. Consequently, designing experiments within underestimated perturbation levels may produce decisions that appear safer than they truly are. We quantify the uncertainty through the posterior probability of robust optimality for a new design $\tilde \xi = (\tilde\xi_t,\tilde\xi_a)$,
\begin{equation}\label{pi}
\pi_d^{\mathrm{rob}}
=
\mathbb{P}\!\left(
d = d^*(\tilde\xi_t,\tilde\xi_a;\epsilon)
\mid 
D
\right),  
\end{equation}

which measures how likely each decision is to remain optimal under
worst-case adversarial perturbations induced by $\epsilon$. We define $\mathcal{D}_{\text{best}}^{\mathrm{rob}}(\tilde \xi; \epsilon)$ as a discrete random variable with probability mass function $\{ \pi_d^{\mathrm{rob}}\}_{d=1}^{|\mathcal{S}|}$, where $\pi_d^{\mathrm{rob}}$ is the posterior probability that $d$ is the robust optimal decision for a new design $\tilde \xi$ under $\epsilon$. 

To reduce uncertainty in the posterior distribution $\mathcal{D}_{\text{best}}^{\mathrm{rob}}(\tilde{\xi}; \epsilon)$, we adopt an information-theoretic query strategy based on the EIG criterion in \cref{eq:4}. Specifically, given the current unlabeled query pool $D_q={(\xi_j,d_j)}$, we define the proposed \textbf{Adversarially Robust Decision Expected Information Gain (AR-DEIG)} for a new design $\tilde{\xi}$ as

\begin{equation}
\label{eq:ad-eig}
\begin{aligned}
\mathrm{AR\text{-}DEIG}(\xi^*,d^*)
=\arg\max_{(\xi_j,d_j)\in D_q}(
H[p(\mathcal{D}_{\text{best}}^{\mathrm{rob}}(\tilde \xi; \epsilon) \mid D)]\\
-
\mathbb{E}_{y_j \sim p(y_j \mid (\xi_j,d_j), D)}
[
H[p(\mathcal{D}_{\text{best}}^{\mathrm{rob}}(\tilde \xi; \epsilon) \mid D \cup \{(\xi_j,d_j,y_j)\})
]]).
\end{aligned}
\end{equation}

\begin{lemma}
    As $\epsilon \to 0$, the AR-DEIG in \cref{eq:ad-eig} becomes equivalent to the D-EIG in \cref{eq:4}.
\end{lemma}
\begin{proof}
The value of $\xi'_a$ lies within $\mathcal{A}_{\epsilon}(\xi_a)$. As $\epsilon \to 0$,
$\mathcal{A}_{\epsilon}(\xi_a) = [\xi_a, \xi_a] = \{\xi_a\}$. This collapses \cref{eq:ad-eig} to minimize over $\{\xi_a\}$. \cref{pi} proves to be the same as \cref{pi_deig} and \cref{eq:ad-eig} becomes \cref{eq:4}.
\end{proof}
\section{Experiments}
We compare our proposed method, denoted as \texttt{AR-DEIG (ours)}, against five active learning baselines, specifically:
\begin{enumerate}
    \item \textbf{Random Sampling} (\texttt{RS}): query points are drawn uniformly at random from the input space.
    
    \item \textbf{Uncertainty Sampling} (\texttt{US}): query points are selected by maximizing the posterior predictive variance.
    
    \item \textbf{Standard Expected Information Gain} (\texttt{PEIG}) \citep{houlsby2011bayesian}: query points are chosen to maximize the expected information gain about the predictive outcome.
    
    \item \textbf{Targeted Expected Information Gain} (\texttt{TEIG}): query points are selected to maximize the expected information gain about a specific target point.
    
    \item \textbf{Decision-based Expected Information Gain} (\texttt{DEIG}) \citep{filstroff2024targeted}: query points are chosen to minimize the entropy of the optimal decision.
\end{enumerate}
To approximate the expectation in \cref{eq:ad-eig}, we use the Monte Carlo \citep{james1980monte} and Gaussian-Hermite quadrature \citep{liu1994note} approximation schemes discussed in \cref{implement}. The code of our proposed method is available at
\href{https://github.com/haripriyaaharikumar/Adversarially-Robust-DEIG}
{\textcolor{linkblue}{https://github.com/haripriyaaharikumar/AR-DEIG}}.
\begin{figure*}
    \centering
    \includegraphics[width=0.95\linewidth]{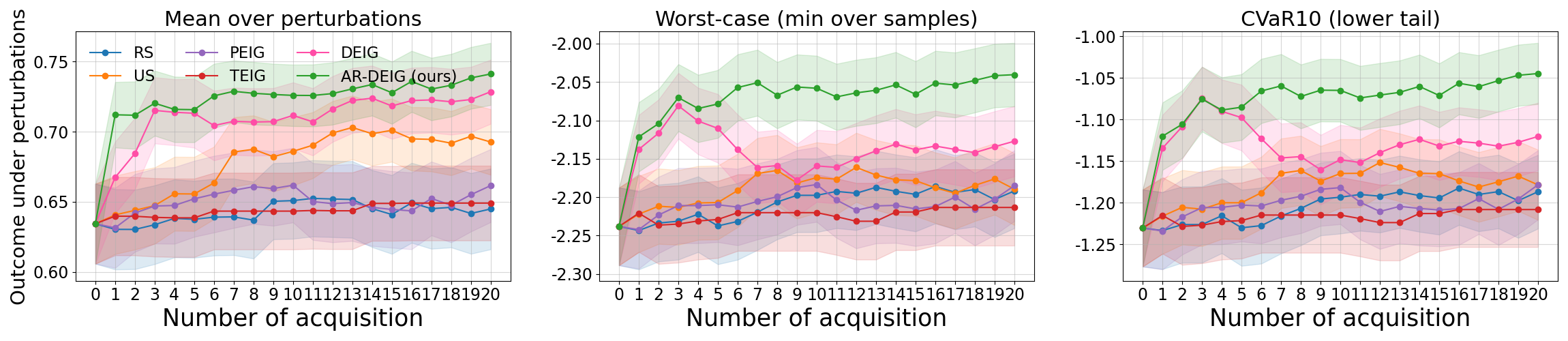}
    \caption{Adversarial Robustness Evaluation: AR-DEIG achieves the strongest robustness across acquisition steps, outperforming baselines in both average and tail-risk metrics. This is reflected by the green curves attaining higher values in the mean panel and less negative values in the worst-case and CVaR10 panels across the acquisition steps. Metrics are computed over 5,000 adversarial samples with perturbation budget $\epsilon$ = 0.3. The x-axis shows the number of acquisitions; mean (left), worst-case (middle), and CVaR10 (right).}
    \label{fig:2}
\end{figure*}
\subsection{Evaluation metrics}\label{evalmetrics}
\subsubsection{Adversarial Robustness Evaluation}\label{adv_rob}
\paragraph{Mean over perturbations}We generated $N$ adversarial perturbations for each of the $M$ test instances by systematically varying inputs along adversarial directions. Each perturbed sample was evaluated using the decision function associated with its original test instance. Performance was then aggregated by averaging outcomes across all $N$ adversarial samples at each acquisition step.
\paragraph{Worst case}For each test instance, the minimum (worst-case) outcome was identified across its $N$ adversarial perturbations. These per-instance worst-case values were then aggregated over all test data, and the mean and standard error were computed at each acquisition step.
\paragraph{CVaR bottom 10 percentile}
For each test instance, $N$ adversarial perturbations were generated, and the outcomes in the lowest $10^{\mathrm{th}}$ percentile were selected. These tail outcomes were then aggregated across all test data, and the mean and standard error were computed at each step.
\subsubsection{Nominal Evaluation}\label{normal}
This is applicable when ground-truth decisions are available.
\paragraph{Accuracy}
For each of the $M$ test instances, a binary accuracy indicator is defined, taking the value 1 when the predicted decision matches with the ground-truth decision, and 0 otherwise.
\begin{figure*}
    \centering
    \includegraphics[width=0.95\linewidth]{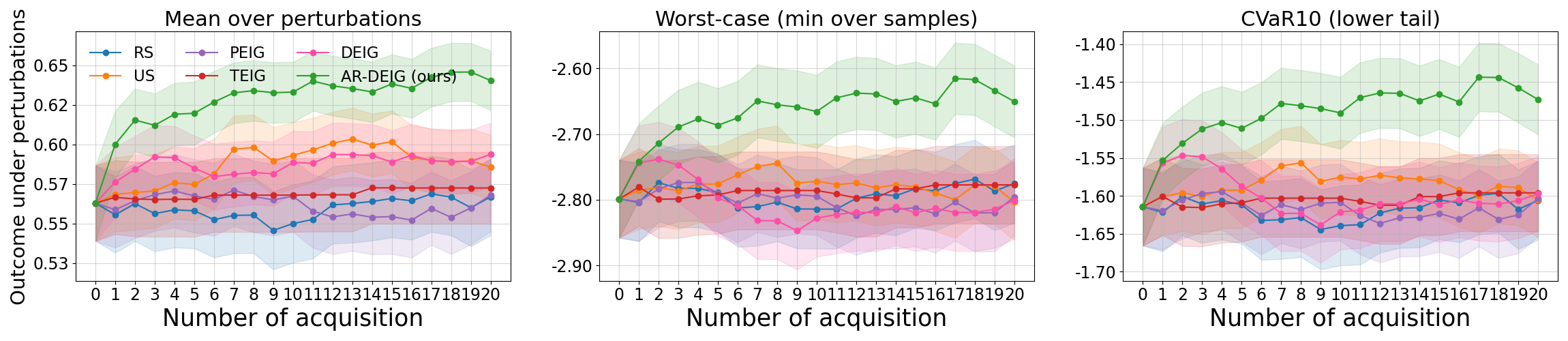}
    \caption{Adversarial Robustness Evaluation: AR-DEIG maintains superior robustness under stronger perturbations, with a larger gap over baselines, particularly in worst-case and CVaR10 metrics. This is evident as the green curves remain consistently higher (or less negative) than others across all panels as acquisitions increase. Metrics are computed over 5,000 adversarial samples with perturbation budget $\epsilon$ = 0.5. The x-axis shows the number of acquisitions; mean (left), worst-case (middle), and CVaR10 (right).}
    \label{fig:3}
\end{figure*}
\begin{figure*}[htbp]
    \begin{subfigure}[b]{0.50\textwidth}
        \centering
        \includegraphics[width=\textwidth]{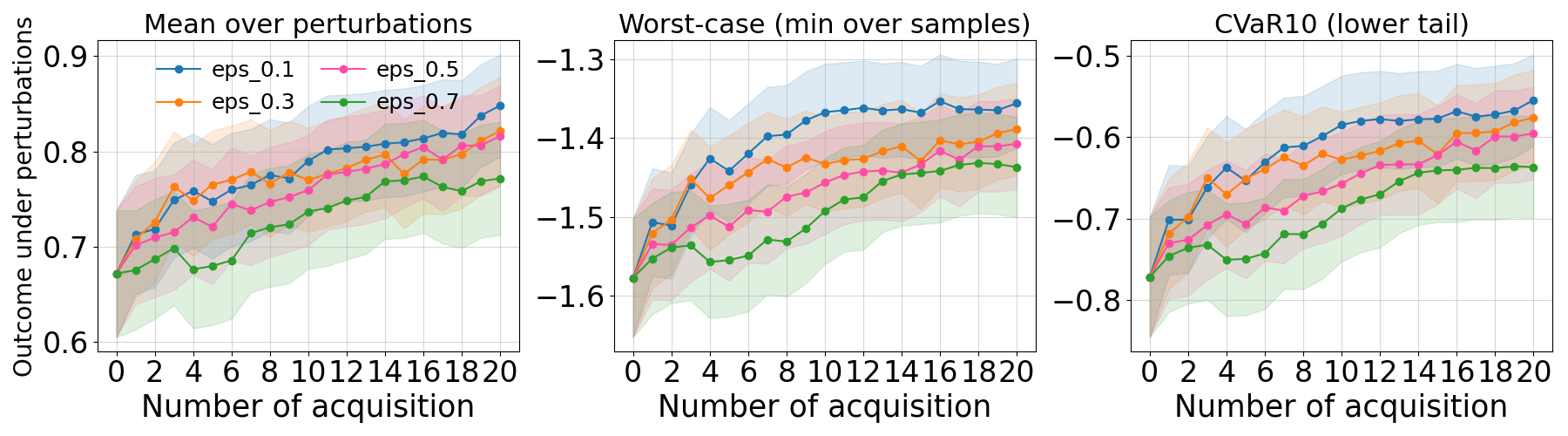}
        \caption{$\epsilon_{eval}=0.1$}
        \label{small}
    \end{subfigure}
    \begin{subfigure}[b]{0.50\textwidth}
        \centering
        \includegraphics[width=\textwidth]{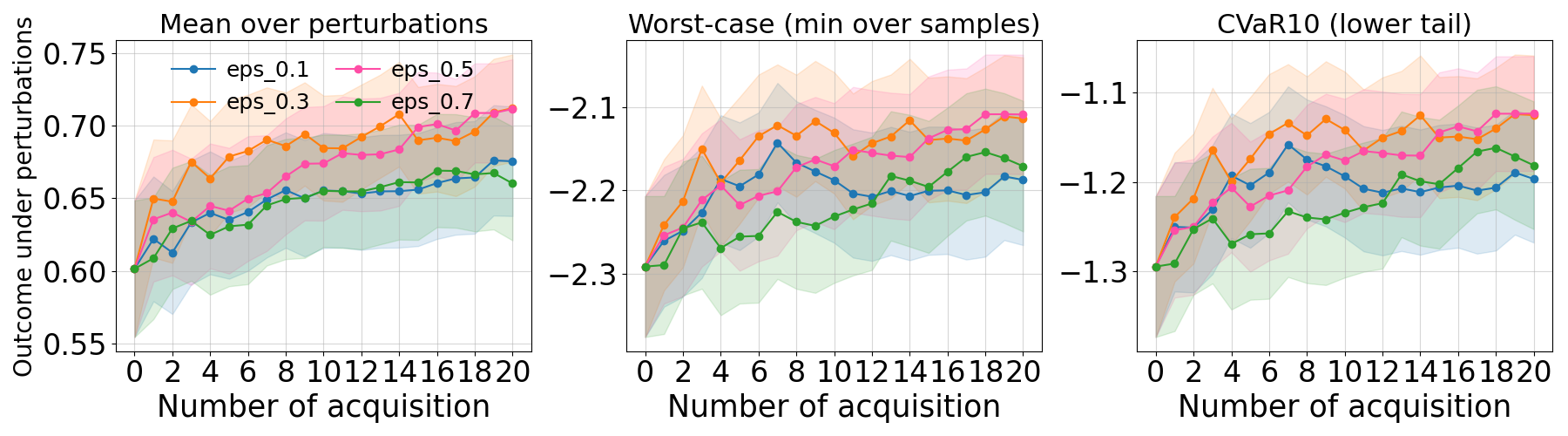}
        \caption{$\epsilon_{eval}=0.3$}
        \label{large}
    \end{subfigure}
\caption{Cross-$\epsilon$ robustness evaluation ($\epsilon_{eval}$ vs. $\epsilon$) for AR-DEIG. We vary the AR-DEIG acquisition budget $\epsilon\in\{0.1,0.3,0.5,0.7\}$ and evaluate the selected decisions under fixed evaluation budgets: $\epsilon_{eval}=0.1$ (left) and $\epsilon_{eval}=0.3$ (right). The x-axis shows the number of acquisitions Vs mean (left), worst-case (middle), and CVaR10 (right) in (a) and (b). Smaller $\epsilon$ performs best under mild evaluation shifts, while intermediate $\epsilon$ values are stronger under larger evaluation shifts.}
    \label{fig:crossepsilon}
\end{figure*}
\begin{figure*}
    \begin{subfigure}{0.48\linewidth}
        \centering
        \begin{minipage}{0.48\linewidth}
            \includegraphics[width=\linewidth]{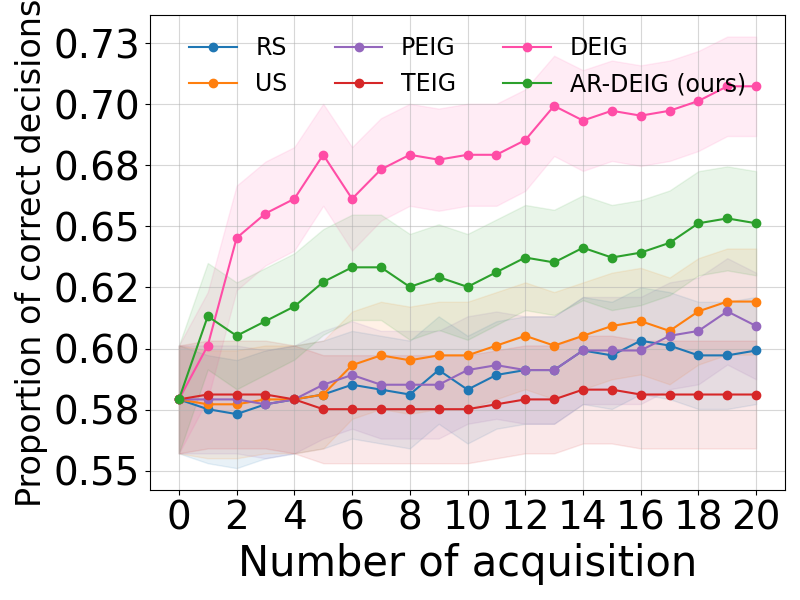}
            \caption*{$\epsilon=0.3$}
        \end{minipage}
        \hfill
        \begin{minipage}{0.48\linewidth}
            \includegraphics[width=\linewidth]{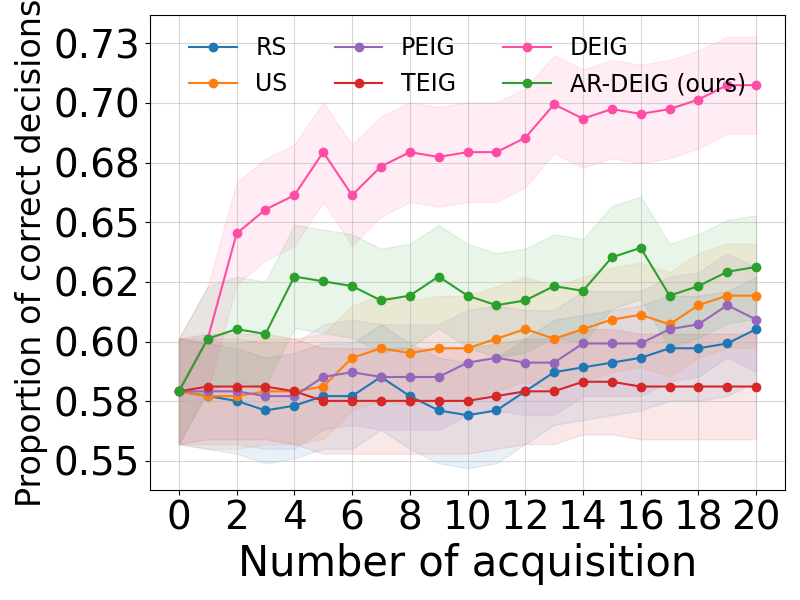}
            \caption*{$\epsilon=0.5$}
        \end{minipage}
        \caption{Accuracy}
        \label{fig:acc_group}
    \end{subfigure}
    \hfill
    \begin{subfigure}{0.48\linewidth}
        \centering
        \begin{minipage}{0.48\linewidth}
            \includegraphics[width=\linewidth]{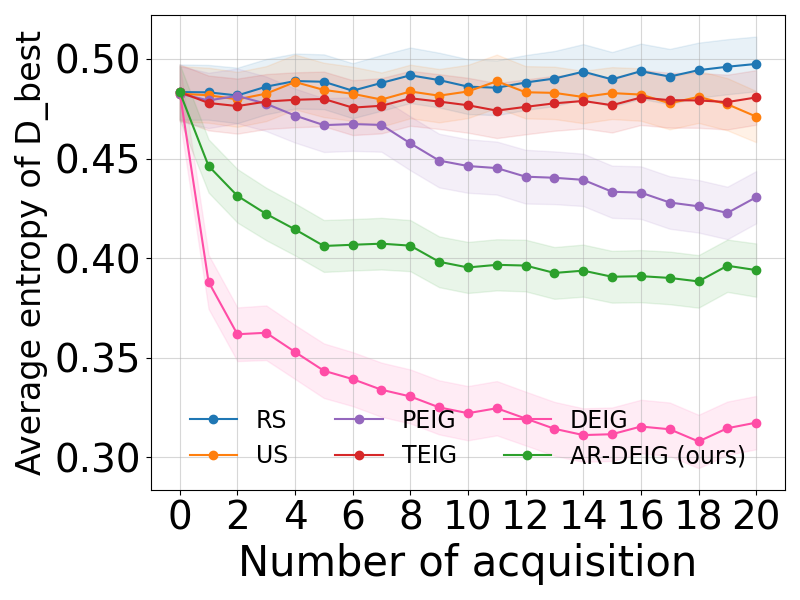}
            \caption*{$\epsilon=0.3$}
        \end{minipage}
        \hfill
        \begin{minipage}{0.48\linewidth}
            \includegraphics[width=\linewidth]{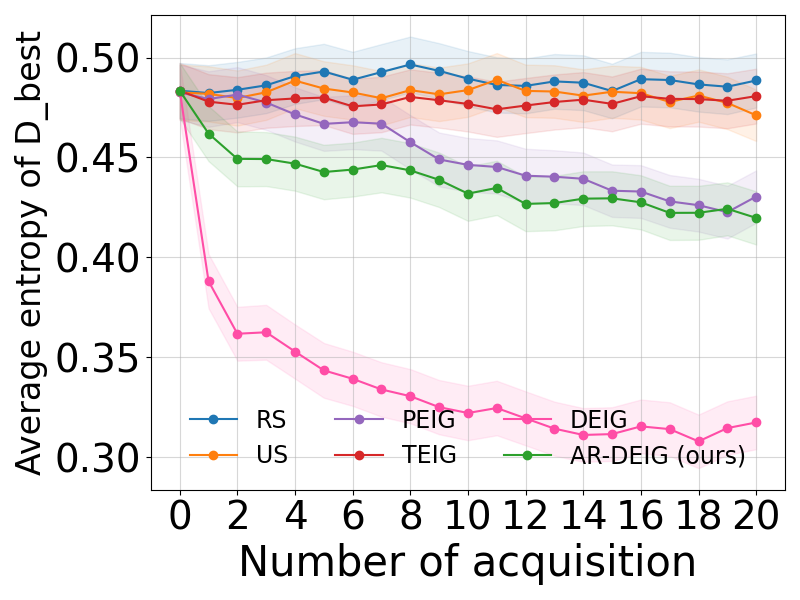}
            \caption*{$\epsilon=0.5$}
        \end{minipage}
        \caption{Entropy}
        \label{fig:ent_group}
    \end{subfigure}
    \caption{Nominal Evaluation: DEIG achieves the best performance, attaining higher accuracy and significantly lower predictive entropy across acquisition steps, while AR-DEIG remains competitive. This is evident as the pink curves dominate in the accuracy panels (left) and consistently achieve the lowest values in the entropy panels (right), with the green curves (AR-DEIG) generally tracking closely behind leading methods. Results are computed on 500 test points without adversarial perturbations, with acquisition performed under budgets $\epsilon=0.3$ and $\epsilon=0.5$ for AR-DEIG. The x-axis denotes the number of acquisitions; panels show accuracy (a) and entropy (b), each comparing the two perturbation levels.}
    \label{fig:7all}
\end{figure*}
\paragraph{Entropy of the posterior}
We compute the entropy of the posterior decision for each of the $N$ test data at each acquisition step.
\begin{figure}[t]
    \begin{subfigure}{0.48\linewidth}
        \includegraphics[width=\linewidth]
     {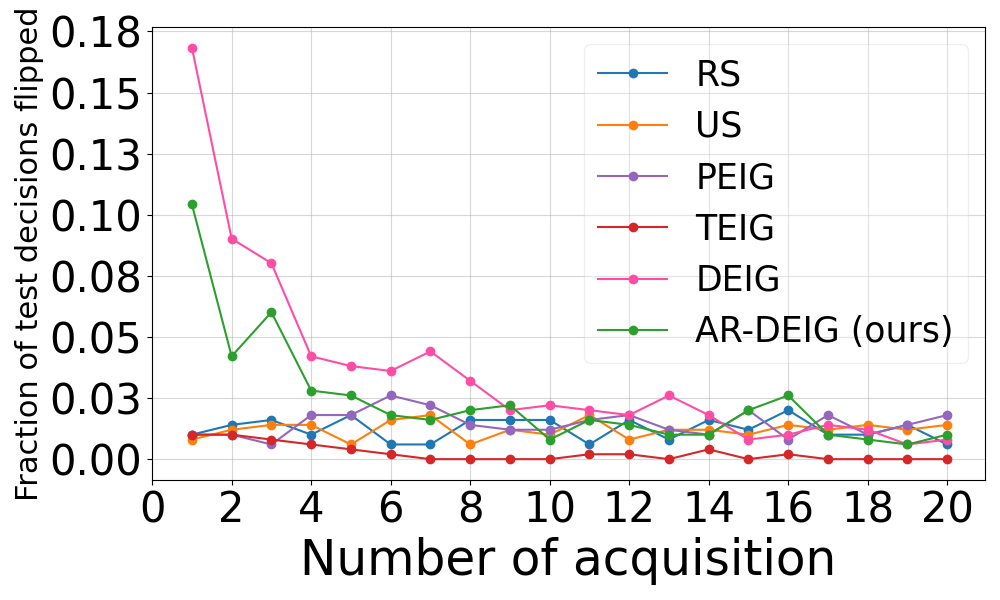}
        \caption{$\epsilon$ = 0.3}
        \label{fig:6b}
    \end{subfigure}
    \hfill
    \begin{subfigure}{0.48\linewidth}
        \includegraphics[width=\linewidth]{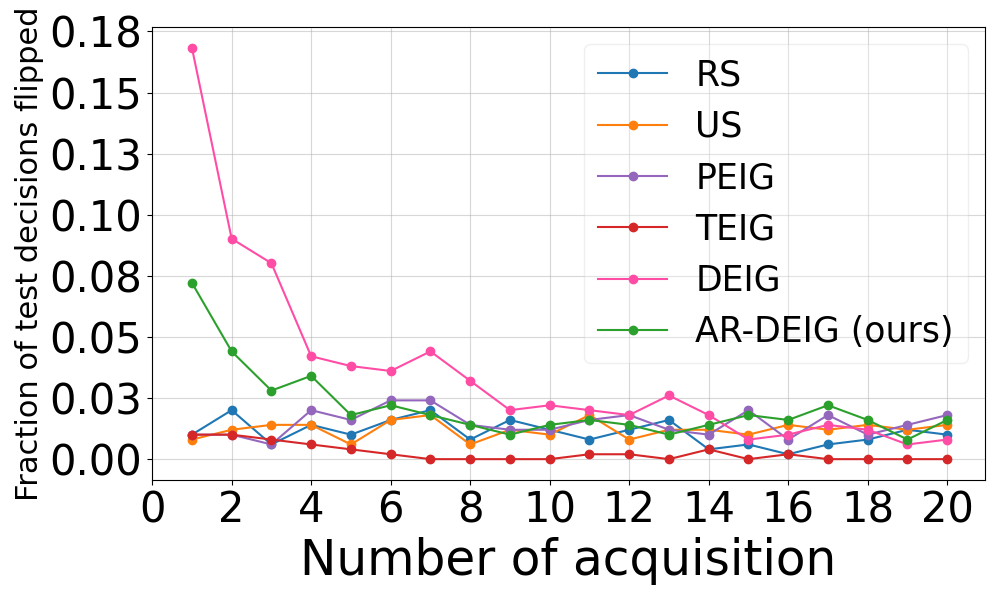}
        \caption{$\epsilon$ = 0.5}
        \label{fig:6c}
    \end{subfigure}
    \caption{Decision flips over acquisitions: Higher flip rates indicate frequent reassignment of test instances to different decision curves, reflecting instability and correlating with degraded adversarial performance.}
    \label{fig:6flips_main}
\end{figure}
\subsection{Synthetic 1-d regression}
\subsubsection{Data generation}\label{synthetic1d}
We generated 1-dimensional synthetic data with 100 training data, 299 data in the query pool, and 500 test data with three decision options. To mimic real-world decision-making scenarios where some outcomes are smooth while others exhibit irregular or rapidly varying behavior, we construct a synthetic dataset with heterogeneous smoothness across decision outcomes. The covariates $x$ are sampled from a standard normal distribution. We generate three different outcomes (decisions) as independent realizations of Gaussian Process (GPs). We set one decision outcome to be smoother than the rest (an example is shown in \cref{fig:sample} as blue (smooth), orange, and green curves). To be more specific, we use Matérn Kernel for this dataset. For the smooth decision $d=0$, we set $ \ell_{\text{smooth}}=0.6$ and $\nu_{\text{smooth}} = 5/2$ while for the remaining decisions ($d=1,2$) we fix $\ell_{\text{rough}}=0.18$ and $\nu_{\text{rough}} = 1/2$. 
\begin{figure*}[t]
    \centering
    \begin{subfigure}{0.30\linewidth}
        \includegraphics[width=\linewidth]{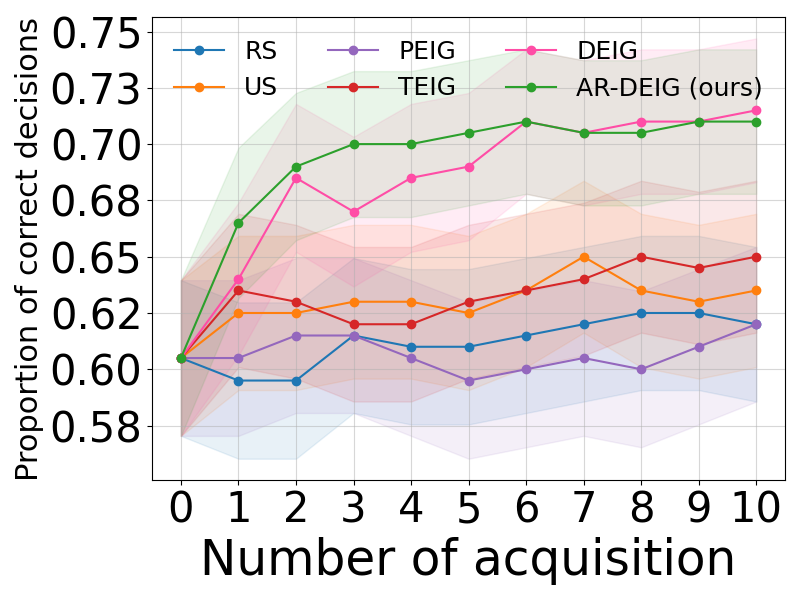}
        \caption{$\epsilon$ = 0.1}
        \label{fig:7a}
    \end{subfigure}
    \hfill
    \begin{subfigure}{0.30\linewidth}
        \includegraphics[width=\linewidth]
{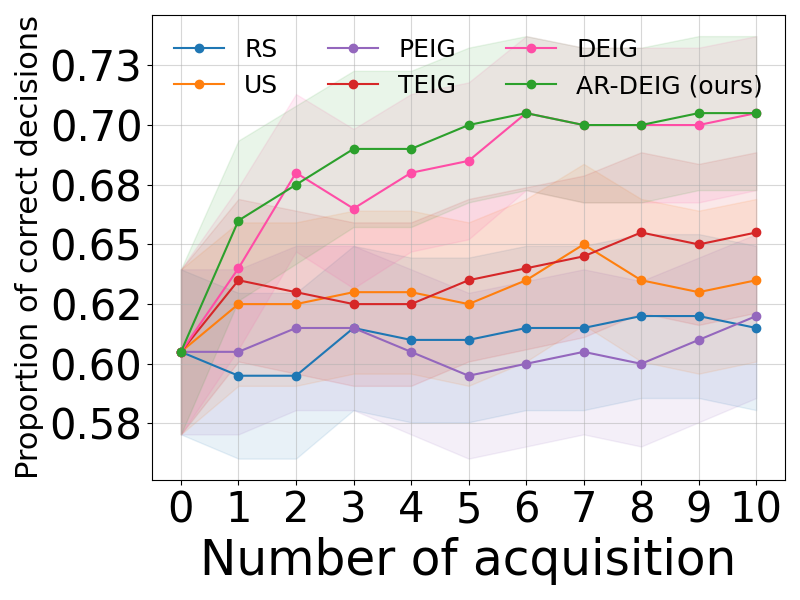}
        \caption{$\epsilon$ = 0.3}
        \label{fig:7b}
    \end{subfigure}
    \hfill
    \begin{subfigure}{0.30\linewidth}
        \includegraphics[width=\linewidth]{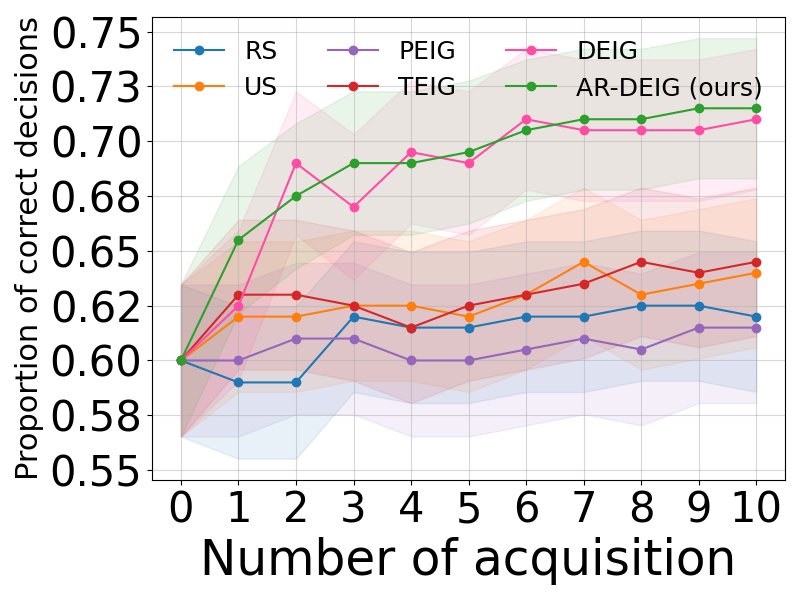}
        \caption{$\epsilon$ = 0.5}
        \label{fig:7c}
    \end{subfigure}
    \caption{Recovery of the ground-truth robust-optimal decision versus number of acquisitions (one adversarial dimension out of five). Ground-truth decisions are computed via worst-case evaluation over adversarial perturbations. AR-DEIG consistently outperforms baseline methods across $\epsilon \in  \{0.1, 0.3, 0.5\}$, demonstrating improved identification of robust decisions.}
    \label{fig:7advs}
\end{figure*}
\begin{figure*}[t]
    \centering
    \begin{subfigure}{0.30\linewidth}
        \includegraphics[width=\linewidth]{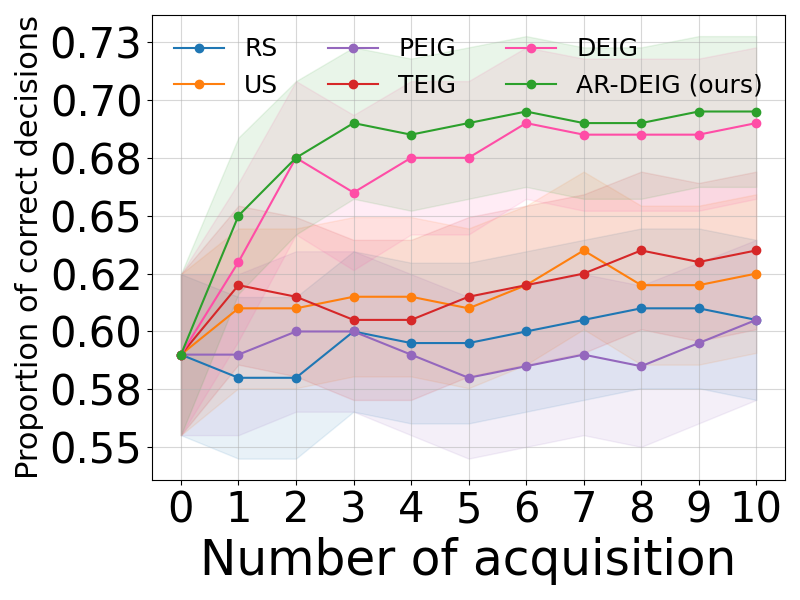}
        \caption{$\epsilon$ = 0.1}
        \label{fig:8a}
    \end{subfigure}
    \hfill
    \begin{subfigure}{0.30\linewidth}
        \includegraphics[width=\linewidth]
{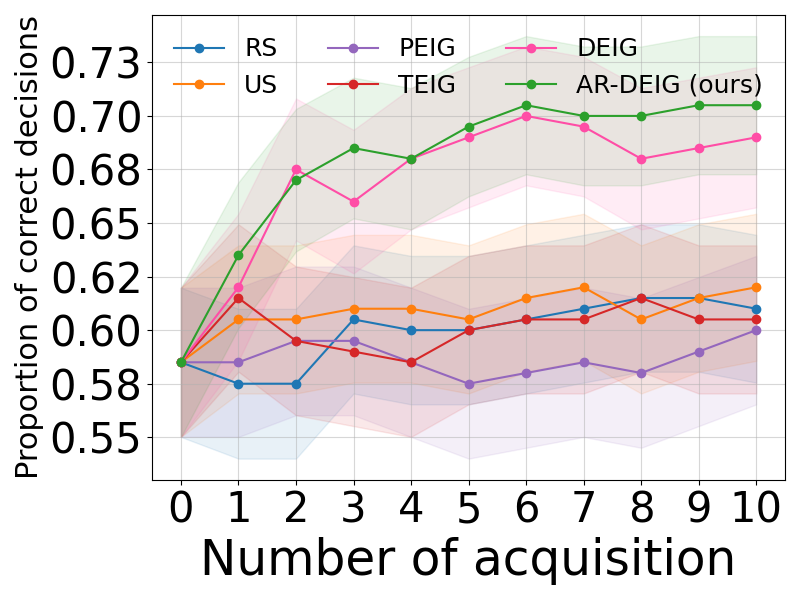}
        \caption{$\epsilon$ = 0.3}
        \label{fig:8b}
    \end{subfigure}
    \hfill
    \begin{subfigure}{0.30\linewidth}
        \includegraphics[width=\linewidth]
        {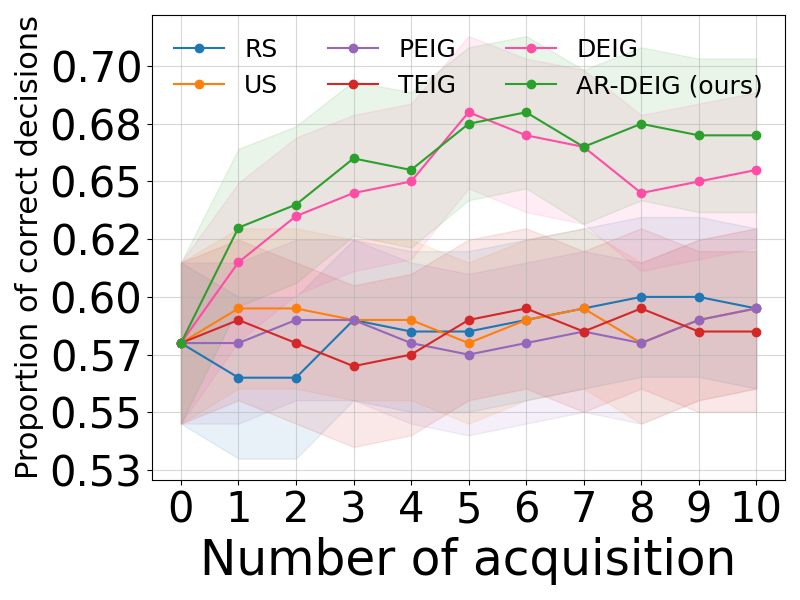}
        \caption{$\epsilon$ = 0.5}
        \label{fig:8c}
    \end{subfigure}
    \hfill
    \caption{Recovery of the ground-truth robust-optimal decision versus number of acquisitions (three adversarial dimensions out of five). Accuracy is computed with respect to robust-optimal decisions defined via worst-case outcomes over adversarial perturbations. As the dimensionality of adversarial variation increases, AR-DEIG maintains superior performance across $\epsilon \in  \{0.1, 0.3, 0.5\}$, indicating improved scalability in identifying robust decisions.}
    \label{fig:8advs}
\end{figure*}
\begin{figure*}
    \centering
    \includegraphics[width=0.95\linewidth]{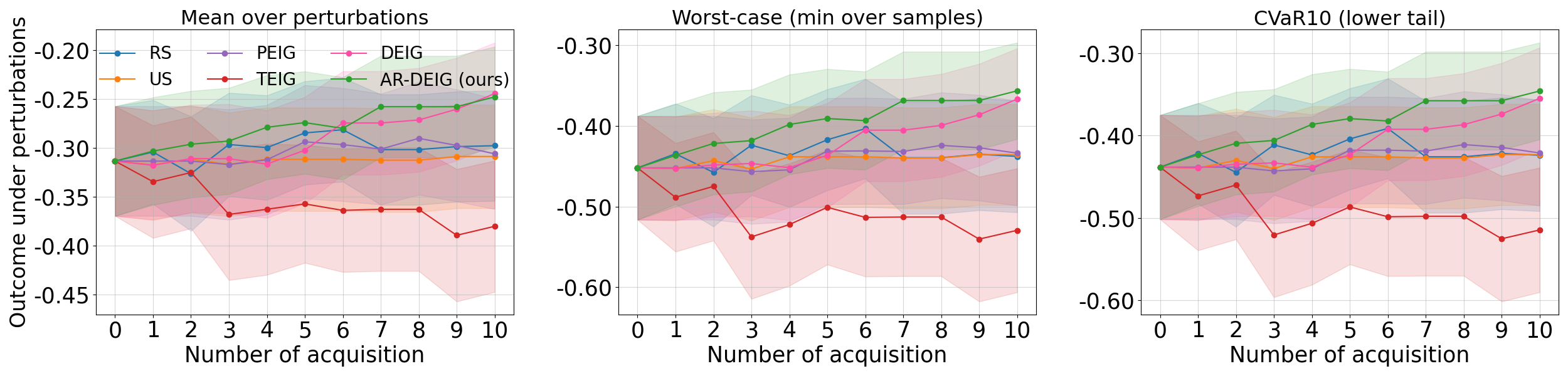}
    \caption{Adversarial Robustness Evaluation: AR-DEIG achieves consistently favorable robustness across acquisition steps, exceeding baselines in both average and tail-risk metrics. This is reflected by the green curves attaining better values in the mean panel and competitive-to-better outcomes in the worst-case and CVaR10 panels. Results are computed over 5,000 adversarial samples with perturbation budget $\epsilon$ = 0.7. Panels show mean (left), worst-case (middle), and CVaR10 (right).}
    \label{fig:oaiworstcase_main}
\end{figure*}
\subsubsection{Adversarial Robustness Evaluation}
We evaluated the reliability of the decisions for 500 test data by measuring the proportion of correct outcomes in 5000 adversarial samples. Specifically, adversarial samples were generated such that they fulfill the assumptions in \cref{ar-deig}.

We report the mean, worst-case, and CVaR10 (bottom 10\%) (metrics defined in \cref{adv_rob}) performance for these adversarial samples with respect to the acquisition steps, in \cref{fig:2} and \cref{fig:3} with $\epsilon$ values as $0.3$ and $0.5$ ($\epsilon$ values $0.1$ and $0.7$ is in \cref{adv_rob_eval}). Across both perturbation levels ($0.3$ and $0.5$), AR-DEIG consistently achieves the strongest robustness, outperforming all five baseline methods in both average and tail-risk metrics. The gains are more pronounced under stronger perturbations i.e for $\epsilon=0.3$, with clear improvements in worst-case and CVaR10 performance. In contrast, D-EIG exhibits a collapse (\cref{fig:2}), which can be attributed to overfitting to nominally optimal decisions that are not robust under adversarial perturbations (see \cref{fig:acq_deig_1}-\ref{fig:acq_deig_10} in  \cref{acq_steps}).

\paragraph{Adaptive $\epsilon$ for dynamic perturbations} We have evaluated cross-robustness $\epsilon$ analysis decoupling acquisition $\epsilon$ and evaluation budgets ($\epsilon_{eval}$) for AR-DEIG as shown in Fig. \ref{fig:crossepsilon} . AR-DEIG uses $\epsilon \in \{0.1,0.3,0.5,0.7\}$ and is evaluated at fixed $\epsilon_{eval}\in\{0.1,0.3\}$. Smaller $\epsilon$ performs best under mild shifts, while intermediate values, especially $0.3$ and $0.5$, perform better at $\epsilon_{eval}=0.3$.
\subsubsection{Nominal Evaluation}
We evaluate accuracy and entropy on 500 test points under nominal (non-adversarial) conditions, as defined in \cref{normal}. \cref{fig:acc_group} and \cref{fig:ent_group} show accuracy and entropy with respect to the number of acquisitions for 
$\epsilon$ $0.3$ and $0.5$, with additional results for $\epsilon$ that is $0.1$ and $0.7$ are provided in \cref{acc_entr_plots}. DEIG achieves the best performance, attaining higher accuracy and lower predictive entropy across acquisition steps, while AR-DEIG remains competitive.This is expected, as nominal settings favor optimal solutions over robustness to perturbations.
\subsubsection{Decision Flip Rate}
To better understand the degradation in worst-case performance, we analyze the stability of the learned decision along the acquisition trajectory for each test data. Specifically, we measure the decision flip rate, defined as the fraction of test contexts whose selected decision changes between consecutive acquisition steps.

As shown in \cref{fig:6flips_main} (with $\epsilon$ values $0.3$ and $0.5$ - rest is in \cref{dfr}) AR-DEIG consistently exhibits lower decision flip rates compared to DEIG, particularly during the early and mid stages of acquisition. In contrast, DEIG demonstrates substantially higher instability, frequently revising its preferred decision as new observations are incorporated.

This instability indicates that DEIG tends to commit to decisions with narrow utility margins that are sensitive to perturbations. By explicitly accounting for adversarial variation, AR-DEIG have more stable decisions, leading to fewer reversals over the acquisition trajectory.

This decision instability coincides with the observed degradation in worst-case and CVaR performance, indicating that nominal information gain may steer the learner toward decisions that appear optimal under the posterior mean but are highly vulnerable to adversarial perturbations.
\subsection{Decision-aware active learning}
\subsubsection{Data generation}
We generated a dataset of comprising 100 training data, 299 data in query pool and 201 test data of dimensions five and four decisions. The co-variates $x$ are sampled from a standard normal distribution. We generate four different outcomes (decisions) as independent realizations of Gaussian Process with Squared Exponential (SE) kernel with variance sampled from $v \sim 0.5+2.0\times\text{U}(0,1)$) and lengthscales sampled from $l \sim \sqrt|\mathcal{S}|(0.25+0.75 \text{U}(0,1))$, where $\text{U}(0,1)$ denotes a uniform random variable between 0 and 1. These outcomes are then corrupted by Gaussian white noise. Finally, the decision variable associated with each data has been done to mimic the imbalance in treatment assignment.
\subsubsection{Adversarial Robustness Evaluation}
For each test instance, we compute an adversarially robust ground-truth decision. Specifically, we generate 5{,}000 adversarial samples in the neighborhood of the test point, evaluate the minimum outcome for each decision, and select the decision that maximizes this worst-case outcome. This selected decision is treated as the \textit{robust ground-truth decision} and is evaluated using \cref{normal}. We draw samples with the epsilon set as $0.1$, $0.3$, and $0.5$. \cref{fig:7advs} and \cref{fig:8advs} illustrate the accuracy of adversarially robust decisions at acquisition step $10$ across different perturbation levels ($\epsilon = 0.1, 0.3, 0.5$). The results in  \cref{fig:7advs} and \cref{fig:8advs} are shown for settings where the adversarial subspace comprises either one or three dimensions out of a total of five.
\subsection{Real world: Osteoarthritis Initiative (OAI) dataset}
Osteoarthritis is a degenerative joint disease with no cure, making early diagnosis and longitudinal decision-making is critical \citep{katz2021diagnosis,filstroff2024targeted}. Patient follow-ups typically occur at 12, 24, 36, 48, and >48 months (five decisions). The details of this dataset is given in \cref{realdata}. The underlying decision boundaries are often non-smooth due to clinical thresholds, population heterogeneity, and measurement noise, making the system sensitive to small input perturbations.

Adversarial robustness evaluation (\cref{adv_rob}) in \cref{fig:oaiworstcase_main} further shows that AR-DEIG achieves consistently favorable performance across acquisition steps, matching or exceeding baselines in both average and tail-risk metrics. This is reflected in less negative mean outcomes and competitive-to-better worst-case and CVaR10 values (more evaluations can be found in \cref{fig:oai_org_su} and \cref{fig:oaiworstcase_20.0}). \cref{fig:oaiworstcase_20.0} shows AR-DEIG degrades at large epsilon i.e. when $\epsilon=20.0$. This is precisely the expected behaviour for larger $\epsilon$ values. By \cref{prop2}, increasing $\epsilon$ induces more conservative decisions; for large $\epsilon$, the objective becomes overly pessimistic and performance degrades.  

\section{Related Work}
\subsection{Active Learning for Decision Making}
Active learning has been extensively studied in the context of statistical efficiency and sample selection \citep{mccallum1998employing,dasgupta2004analysis,golovin2011adaptive}, and has more recently been extended to decision-centric settings. A key line of work focuses on selecting samples that directly improve downstream decisions rather than predictive accuracy. For instance, \cite{berger2013statistical} proposes minimizing Type S error, i.e., the probability of inferring the incorrect sign of a treatment effect under data imbalance. 

Targeted active learning \citep{filstroff2024targeted} introduces a decision-aware Expected Information Gain (EIG) criterion that prioritizes queries reducing uncertainty in downstream decisions. Similarly, \cite{lacoste2011approximate} emphasize posterior regions that are most relevant for decision-making, rather than uniformly improving parameter estimation. More recently, \cite{huang2024amortized} propose an amortized framework that integrates decision-making objectives directly into experimental design. In a different direction, \cite{bal_etal_2025optimistic} adopt a game-theoretic perspective to address combinatorial and high-dimensional Bayesian optimization. Variational and amortized methods \citep{foster2019variational,foster2021deep} are promising ways to reduce computational cost. The ambiguity-set approach in \cite{go2022robust} is also relevant; however, it would face computational challenges and require estimators in \cite{foster2019variational,foster2021deep}. 

Despite these advances, existing approaches primarily assume a \emph{nominal} setting, where decisions are evaluated under the learned model without accounting for adversarial or worst-case perturbations. Consequently, they may select queries that are informative under the model but fail to improve decision robustness. Furthermore, most methods rely on standard EIG formulations, which either require costly retraining or focus on predictive uncertainty rather than uncertainty over \emph{optimal decisions}.
\subsection{Adversarial Robustness in Bayesian Methods}
Robustness in Bayesian learning has been studied from multiple perspectives. Distributionally robust approaches \citep{kirschner2020distributionally,husain2023distributionally} focus on uncertainty over data-generating distributions, rather than explicit worst-case perturbations. Other works, such as \cite{gloeckler2023adversarial}, develop adversarially robust posterior inference, but do not address the problem of data acquisition or experimental design.

In the context of Gaussian processes, \cite{bogunovic2018adversarially} propose adversarially robust optimization methods that seek solutions stable under input perturbations, which focus on identifying robust optima of the objective function.

Existing robust Bayesian methods primarily address either \emph{inference} or \emph{optimization}, but not \emph{data acquisition}. In particular, they do not provide mechanisms for selecting informative queries that improve robustness of downstream decisions. Moreover, robustness is typically defined at the level of function optimization, rather than at the level of \emph{decision uncertainty}. As a result, there is a lack of principled acquisition strategies that explicitly target the reduction of uncertainty over robust optimal decisions.

In contrast, our approach introduces an adversarially robust decision-centric acquisition function that directly quantifies information gain over the \emph{robust optimal decision}. This bridges the gap between decision-aware active learning and adversarial robustness, enabling query selection that is both informative and robust to perturbations.
\section{Conclusion}
We propose an adversarially robust, decision-aware experimental design framework that explicitly targets the stability of downstream decisions under perturbations in adversarial variables. Building on Bayesian decision theory, we introduced a worst-case utility formulation and derived the AR-DEIG acquisition criterion, which prioritizes reducing uncertainty over the robust-optimal decision rather than optimizing nominal utility alone. This shift aligns the experimental design process with the practical objective of making reliable decisions in the presence of weakly modeled effects.

Empirical evaluations on both synthetic and real-world datasets show that conventional approaches often produce high-confidence yet brittle decisions. In contrast, our method yields more stable outcomes with improved worst-case and tail-risk performance. These results highlight the importance of incorporating adversarial robustness into experimental design to ensure reliable decision-making. Future work will focus on extending robustness notions, improving scalability, and strengthening theoretical guarantees.

\begin{acknowledgements} 
This work was supported by the UKRI Turing AI World-Leading Researcher Fellowship [EP/W002973/1], UKRI AI Hub in Generative Models [EP/Y028805/1], and European Lighthouse of AI for Sustainability [ELIAS, 10080425]. S. Kaski was supported by the Research Council of Finland Flagship programme: Finnish Center for Artificial Intelligence FCAI and decisions 358958, 359567, and 359207. H. Harikumar and S. Kaski were supported by the UKRI Turing AI World-Leading Researcher Fellowship (EP/W002973/1), UKRI AI Hub in Generative Models (EP/Y028805/1), and European Lighthouse of AI for Sustainability (ELIAS, 10080425). S. Katt and S. Kaski were supported by EU funding ERC ODD-ML 101201120. YZ. Barlas was supported by a departmental studentship at The University of Manchester. The authors also thank Daolang Huang for the discussions along the progress of the work and Jaeyoung Lee for the initial discussions. The authors thank all anonymous reviewers for their constructive feedback. The
authors also acknowledge the computational resources provided by the Aalto Science-IT project.
\end{acknowledgements}





\bibliography{uai2026-template}

\newpage

\onecolumn

\appendix
\begin{center}
    \LARGE \textbf{Appendix}
\end{center}
\section{Approximation for Implementation}\label{implement}
\subsection{Expected Information Gain (EIG)}
We begin with a standard Bayesian regression model with likelihood $p(y \mid x,\theta)$ and prior $p(\theta)$, inducing the posterior $p(\theta \mid D)$. The Expected Information Gain (EIG) at a query $x$ is defined as:
\begin{equation}
\mathrm{EIG}(x) = H[p(\theta \mid D)] - \mathbb{E}_{p(y \mid x,D)}\big[H[p(\theta \mid D \cup \{(x,y)\})].
\end{equation}

This expression can be rewritten as the mutual information between $y$ and $\theta$:
\begin{equation}
I(y;\theta \mid x,D) = \iint p(y,\theta \mid x,D)\log\frac{p(y,\theta \mid x,D)}{p(y \mid x,D)p(\theta \mid x,D)}\,dy\,d\theta.
\end{equation}

By symmetry of mutual information, we obtain:
\begin{equation}
\mathrm{EIG}(x) = H[p(y \mid x,D)] - \mathbb{E}_{p(\theta \mid D)}\big[H[p(y \mid x,\theta)]\big],
\end{equation}
which avoids model retraining and operates in the output space.

\paragraph{Gaussian Process Case.}
For a non-parametric regression model
\begin{equation}
y = f(x) + \epsilon, \quad \epsilon \sim \mathcal{N}(0,\sigma^2),
\end{equation}
we obtain:
\begin{equation}
\mathrm{EIG}(x) = H[p(y \mid x,D)] - \mathbb{E}_{p(f \mid D)}\big[H[p(y \mid x,f)]\big].
\end{equation}

Under Gaussian Process (GP) \citep{williams2006gaussian} regression, the predictive posterior is Gaussian with variance $\sigma_x^2 + \sigma^2$, yielding:
\begin{equation}
\mathrm{EIG}(x) = \frac{1}{2}\left(\log(\sigma_x^2 + \sigma^2) - \log(\sigma^2)\right).
\end{equation}

Thus, maximizing EIG reduces to selecting points with high predictive variance $\sigma_x^2$.
\subsection{Gauss--Hermite Quadrature}\label{GP}

We consider expectations of the form
\begin{equation}
\mathbb{E}[f(y)] = \int f(y)\,p(y)\,dy,
\end{equation}
where $y \sim \mathcal{N}(\mu,\sigma^2)$.

Using Gauss--Hermite quadrature \cite{liu1994note} of order $N$, this expectation can be approximated as
\begin{equation}
\mathbb{E}[f(y)] \approx \frac{1}{\sqrt{\pi}} \sum_{i=1}^{N} \omega_i \, f\big(\sqrt{2}\sigma x_i + \mu\big),
\end{equation}
where $\{x_i\}_{i=1}^N$ are the roots of the Hermite polynomial $H_N(\cdot)$, and the corresponding weights $\{\omega_i\}_{i=1}^N$ are given by
\begin{equation}
\omega_i = \frac{2^{N-1} N! \sqrt{\pi}}{N^2 \left(H_{N-1}(x_i)\right)^2}.
\end{equation}

\paragraph{Remark.}
Gauss--Hermite quadrature provides an efficient deterministic approximation of Gaussian expectations and is commonly used as an alternative to Monte Carlo sampling when evaluating acquisition functions.

\begin{figure}
    \centering
\includegraphics[width=1.0\linewidth]{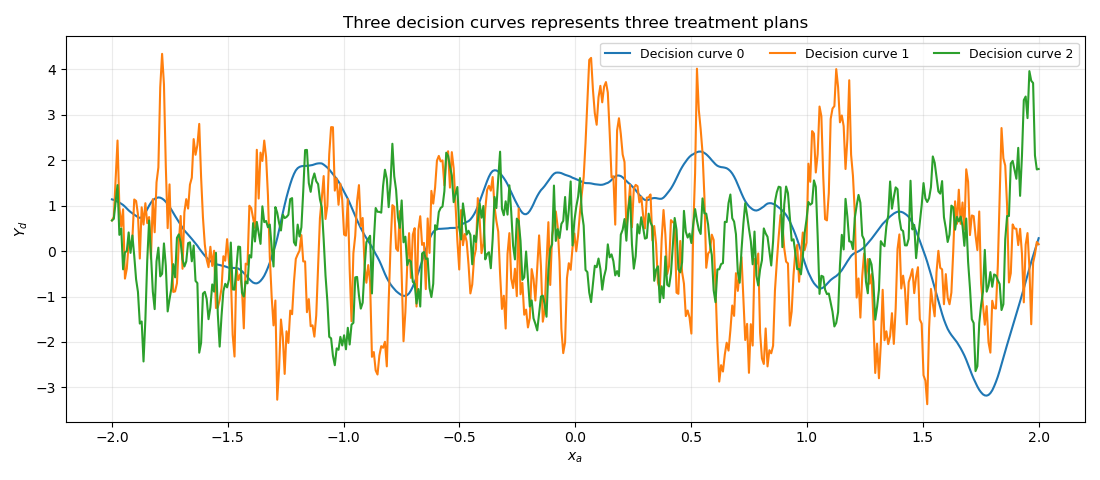}
    \caption{Three treatment plans (shown as green, orange, and blue curves) with their corresponding outcomes on the y-axis. The blue curve is smoother than the orange and green curves, reflecting differences in the underlying generation process described in \cref{synthetic1d}.}
    \label{fig:sample}
\end{figure}

\section{Computational time analysis}
We added runtime measurements to quantify the computational overhead in \cref{comptab}. Acquisition time is the mean time to select one query point (standard deviation in brackets); total time is the runtime over 20 acquisition steps. Specifically, in the 1 dimension setting, AR-DEIG takes 271.31 seconds per acquisition compared with 162.97 seconds for DEIG. In the 20 dimensions setting, AR-DEIG takes 584.96 seconds per acquisition compared with 247.25 seconds for DEIG, about 2.36× slower. The results show that AR-DEIG is computationally heavier, as expected, but still feasible.
\begin{table}[t]
\centering
\begin{tabular}{ c c c }
\hline
Method & Acquisition time (in sec.) & Total time with 20 acquisitions (in sec.) \\
\hline
 RS (1 dimension) & 0.00003 (0.000006) & 14.60 \\ 
 US (1 dimension) & 0.19 (0.01) & 21.94 \\  
 PEIG (1 dimension) & 1.68 (0.04) & 49.10 \\
 TEIG (1 dimension) & 93.27 (4.82) & 1883.97 \\
 DEIG (1 dimension) & 162.97 (7.97) &  3277.67\\
 AR-DEIG (1 dimension) & 271.31 (5.62) & 5441.31\\
 DEIG (20 dimensions) & 247.25 (13.49) & 4959.24\\
 AR-DEIG (20 dimensions) & 584.96 (14.74) & 11712.33\\
\hline
\end{tabular}
\caption{Runtime (in seconds) comparison of baseline methods - Random Sampling (RS), Uncertainty Sampling (US), Standard Expected Information Gain (PEIG), Targeted Expected Information Gain (TEIG), Decision-based Expected Information Gain (DEIG), and our proposed method Adversarially Robust Decision-Expected Information Gain (AR-DEIG). Acquisition time is the mean query-selection time per acquisition step, with standard deviation in bracket; total time is measured over 20 acquisitions.}
\label{comptab}
\end{table}
\section{Additional Experiment Results}\label{add results1}
\subsection{1d-regression}\label{1dmwc}
A sample plot illustrating three decision outcomes is shown in \cref{fig:sample}. The blue curve represents the smoother outcome, while the orange and green curves exhibit greater variability and roughness. We set one decision outcome to be smoother than the rest (an example is shown in the  \cref{fig:sample} as blue (smooth), orange, and green curves). To be more specific, we use Matérn Kernel for this generating this dataset. For the smooth decision $d=0$, we set $ \ell_{\text{smooth}}=0.6$ and $\nu_{\text{smooth}} = 5/2$ while for the remaining decisions ($d=1,2$) we fix $\ell_{\text{rough}}=0.18$ and $\nu_{\text{rough}} = 1/2$.
\begin{figure*}[htbp]
    \hspace{0.05\textwidth}
    \centering
        \includegraphics[width=\textwidth]{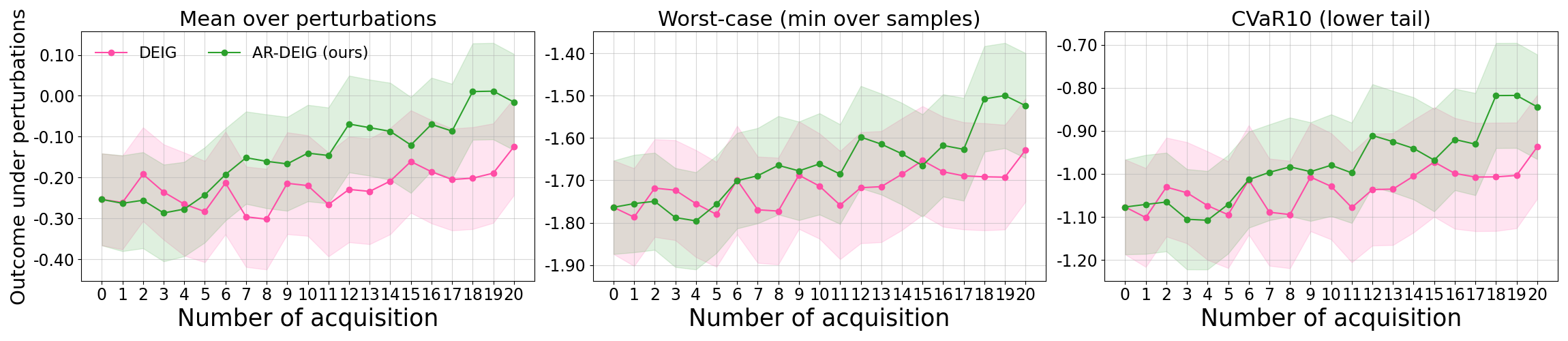}
        \caption{$\epsilon=0.5$}
        \label{small}
    \caption{Adversarial Robustness Evaluation high dimension (12 adversarial variables among total of 20 co-variates) with $\epsilon=0.5$: AR-DEIG maintains superior robustness compared to the  baseline DEIG. This is evident as the green curves remain consistently higher than others across all panels as acquisitions increase. The x-axis shows the number of acquisitions; mean (left), worst-case (middle), and CVaR10 (right).}
    \label{fig:loss-convergence}
\end{figure*}
\subsection{Higher dimensional Synthetic dataset}
To further assess scalability, we added a 20 dimensions experiment with 12 adversarial covariates and 3 decisions. To be more specific, we use Matérn Kernel for this dataset generation. For the smooth decision $d=0$, we set $ \ell_{\text{smooth}}=0.6$ and $\nu_{\text{smooth}} = 5/2$ while for the remaining decisions ($d=1,2$) we fix $\ell_{\text{rough}}=0.18$ and $\nu_{\text{rough}} = 1/2$. The results in \cref{fig:loss-convergence} show that our proposed method AR-DEIG outperforms baseline DEIG when $\epsilon=0.5$. 
\subsubsection{Adversarial Robustness Evaluation} \label{adv_rob_eval}
The mean, worst-case, and CVaR10 (bottom 10\%) (metrics defined in \cref{evalmetrics}) performance for these adversarial samples with respect to the acquisition steps, in \cref{fig:5_suppl} and \cref{fig:6_suppl} for $\epsilon$ values 0.1 and 0.7.
\begin{figure*}
    \centering
\includegraphics[width=0.95\linewidth]{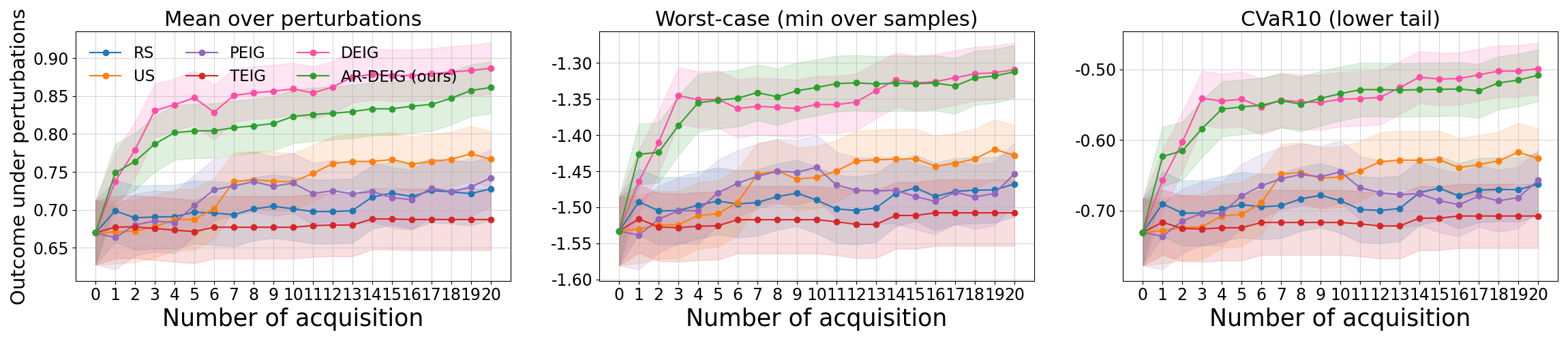}
    \caption{Adversarial Robustness Evaluation: AR-DEIG achieves the slightly better robustness across acquisition steps, compared to baselines in both average and tail-risk metrics. This is reflected by the green curves attaining higher values in the mean panel and less negative values in the worst-case and CVaR10 panels across the acquisition steps. Metrics are computed over 5,000 adversarial samples with perturbation budget $\epsilon$ = 0.1. The x-axis shows the number of acquisitions; mean (left), worst-case (middle), and CVaR10 (right).}
    \label{fig:5_suppl}
\end{figure*}
\subsubsection{Nominal Evaluations}\label{acc_entr_plots}
Accuracy and entropy are shown over the number of acquisitions for the 500 test points, illustrating how performance and uncertainty evolve as more queries are collected. Results for $\epsilon \in {0.1, 0.7}$ are presented in \cref{fig:acc_group_suppl}–\cref{fig:ent_group_suppl}.
\begin{figure*}
    \centering
\includegraphics[width=0.95\linewidth]{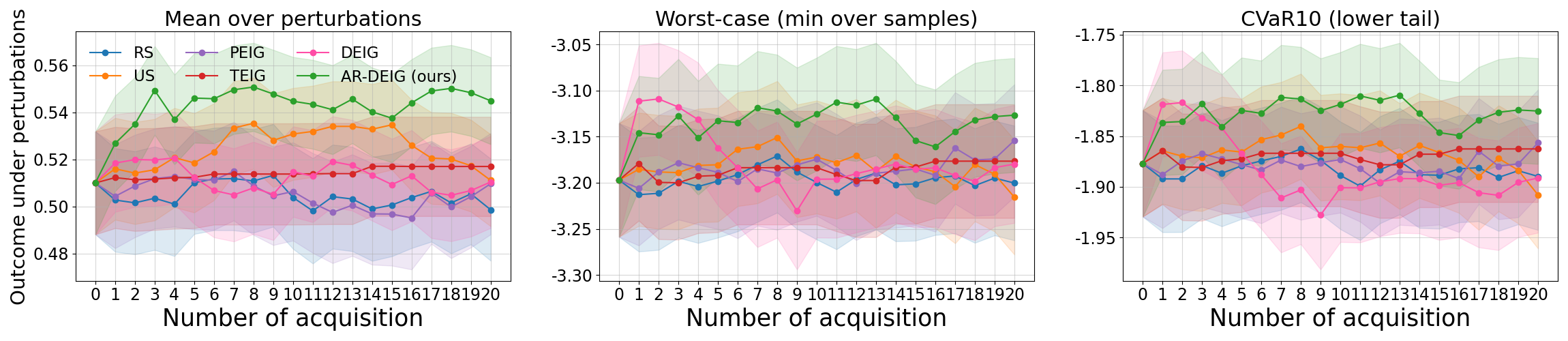}
    \caption{Adversarial Robustness Evaluation: AR-DEIG achieves the strongest robustness across acquisition steps, outperforming baselines in both average and tail-risk metrics. This is reflected by the green curves attaining higher values in the mean panel and less negative values in the worst-case and CVaR10 panels across the acquisition steps. Metrics are computed over 5,000 adversarial samples with perturbation budget $\epsilon$ = 0.7. The x-axis shows the number of acquisitions; mean (left), worst-case (middle), and CVaR10 (right).}
    \label{fig:6_suppl}
\end{figure*}
\begin{figure*}
    \begin{subfigure}{0.48\linewidth}
        \centering
        \begin{minipage}{0.48\linewidth}
            \includegraphics[width=\linewidth]{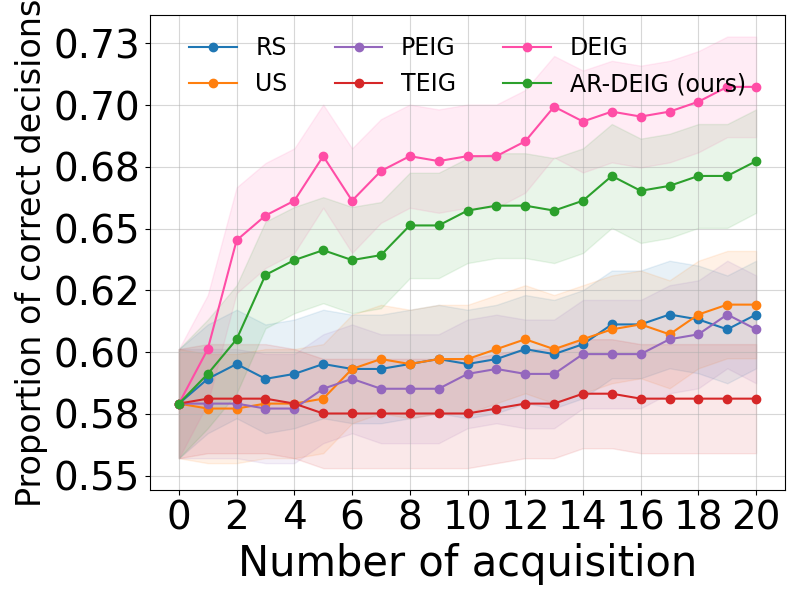}
            \caption*{$\epsilon=0.1$}
        \end{minipage}
        \hfill
        \begin{minipage}{0.48\linewidth}
            \includegraphics[width=\linewidth]{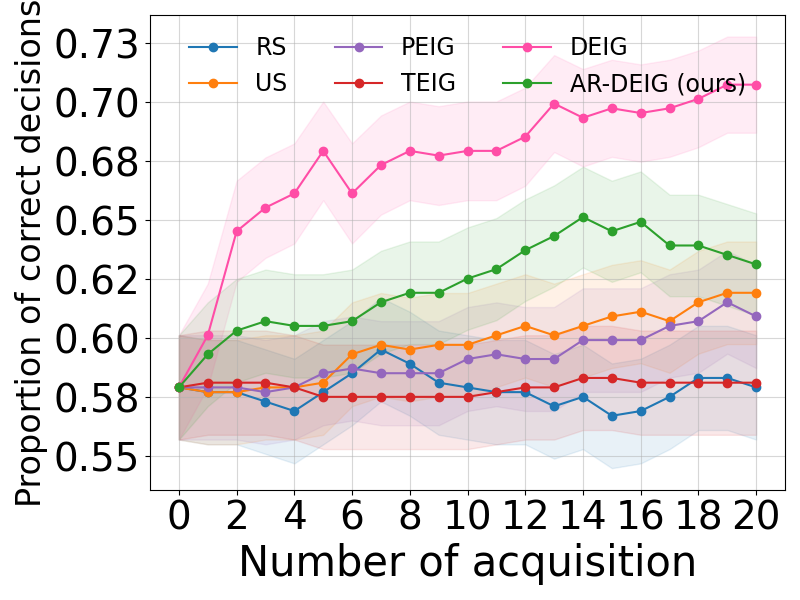}
            \caption*{$\epsilon=0.7$}
        \end{minipage}
        \caption{Accuracy}
        \label{fig:acc_group_suppl}
    \end{subfigure}
    \hfill
    \begin{subfigure}{0.48\linewidth}
        \centering
        \begin{minipage}{0.48\linewidth}
            \includegraphics[width=\linewidth]{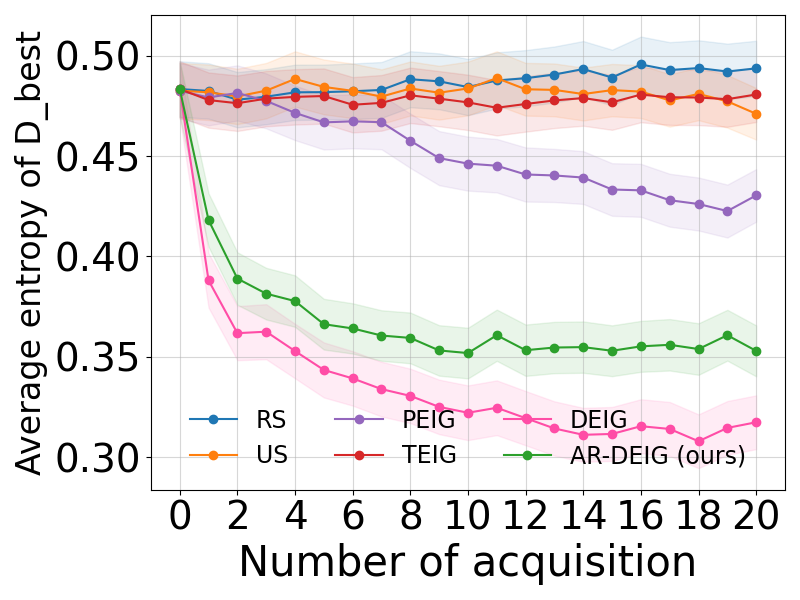}
            \caption*{$\epsilon=0.1$}
        \end{minipage}
        \hfill
        \begin{minipage}{0.48\linewidth}
            \includegraphics[width=\linewidth]{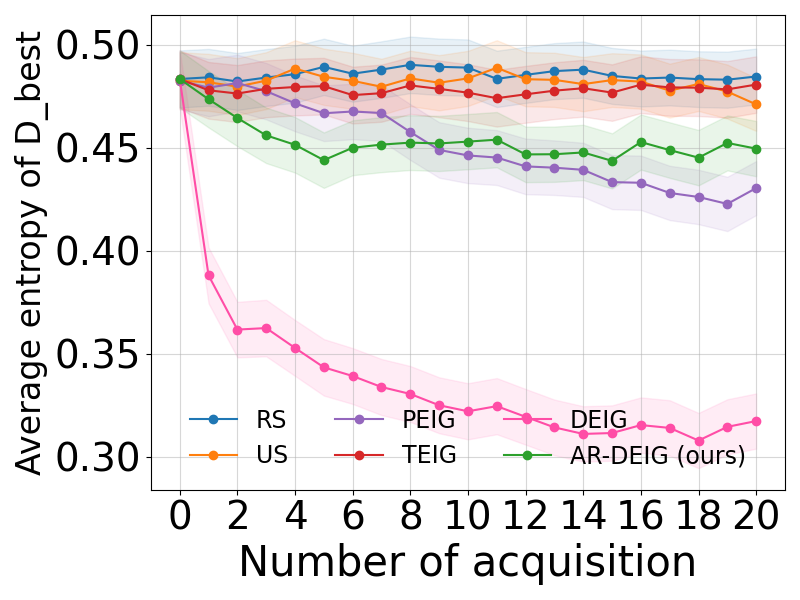}
            \caption*{$\epsilon=0.7$}
        \end{minipage}
        \caption{Entropy}
        \label{fig:ent_group_suppl}
    \end{subfigure}
    \caption{Nominal Evaluation: DEIG achieves the best performance, attaining higher accuracy and significantly lower predictive entropy across acquisition steps, while AR-DEIG remains competitive. This is evident as the pink curves dominate in the accuracy panels (left) and consistently achieve the lowest values in the entropy panels (right), with the green curves (AR-DEIG) generally tracking closely behind leading methods. Results are computed on 500 test points without adversarial perturbations, with acquisition performed under budgets $\epsilon=0.1$ and $\epsilon=0.7$. The x-axis denotes the number of acquisitions; panels show accuracy (a) and entropy (b), each comparing the two perturbation levels.}
    \label{fig:othereps}
\end{figure*}
\subsubsection{Decision Flip Rate}\label{dfr}
The decision flip rates for $\epsilon = 0.1$ and $\epsilon = 0.7$ are presented in \cref{fig:0.1flip} and \cref{fig:0.7flip}, respectively. AR-DEIG consistently achieves lower flip rates than DEIG, particularly during the early and intermediate stages of acquisition. In contrast, DEIG exhibits markedly higher instability, frequently altering its preferred decision as new observations are incorporated.
\begin{figure}[t]
    \begin{subfigure}{0.48\linewidth}
        \includegraphics[width=\linewidth]
     {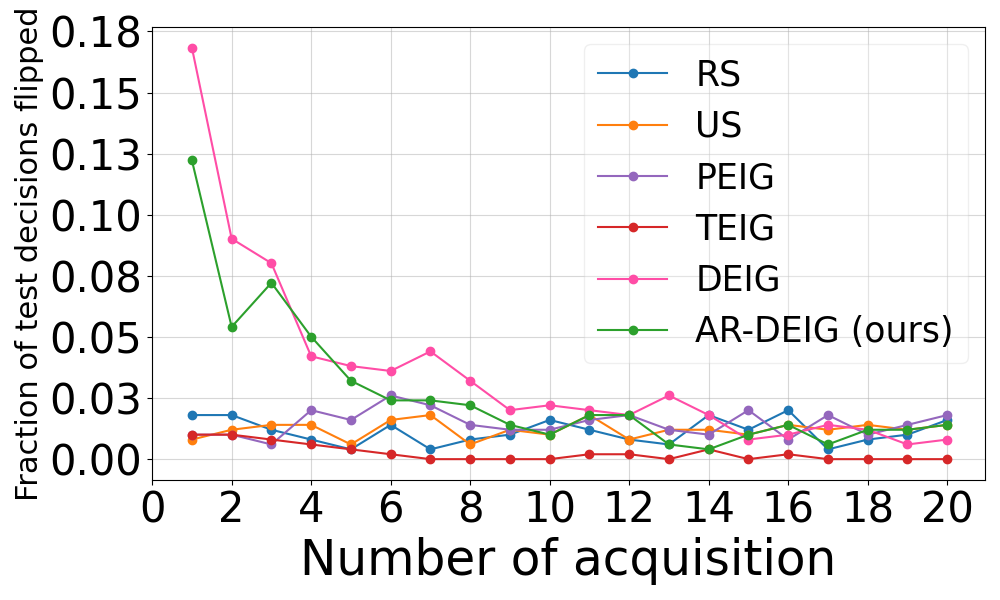}
        \caption{$\epsilon$ = 0.1}
        \label{fig:0.1flip}
    \end{subfigure}
    \hfill
    \begin{subfigure}{0.48\linewidth}
        \includegraphics[width=\linewidth]{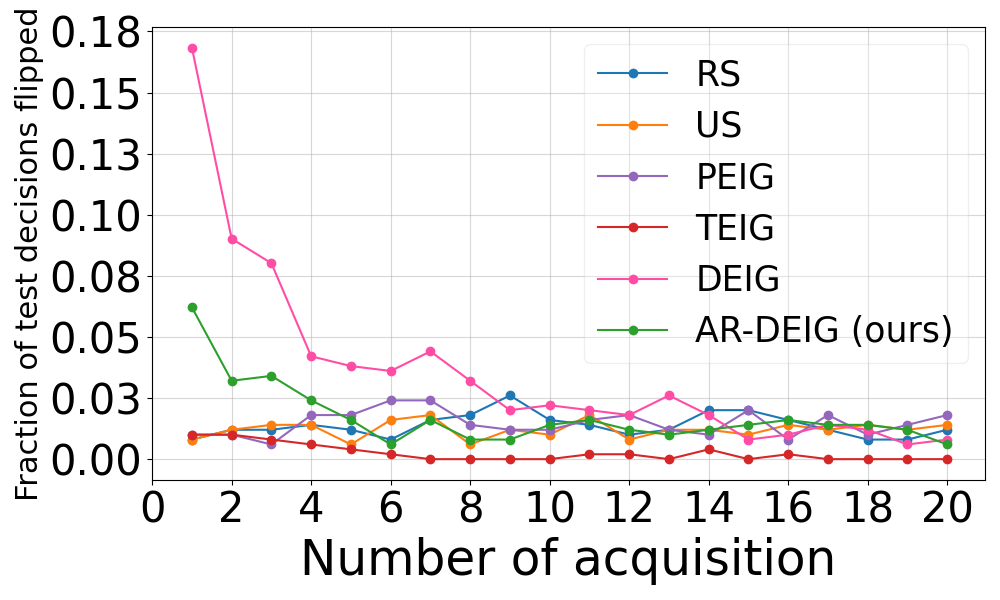}
        \caption{$\epsilon$ = 0.7}
        \label{fig:0.7flip}
    \end{subfigure}
    \caption{Decision flips over acquisitions with $\epsilon=0.1,0.7$: Higher flip rates indicate frequent reassignment of test instances to different decision curves, reflecting instability and correlating with degraded adversarial performance.}
    \label{fig:6flips}
\end{figure}
\begin{figure*}
    \begin{subfigure}{0.48\linewidth}
        \centering
        \begin{minipage}{0.48\linewidth}
            \includegraphics[width=\linewidth]{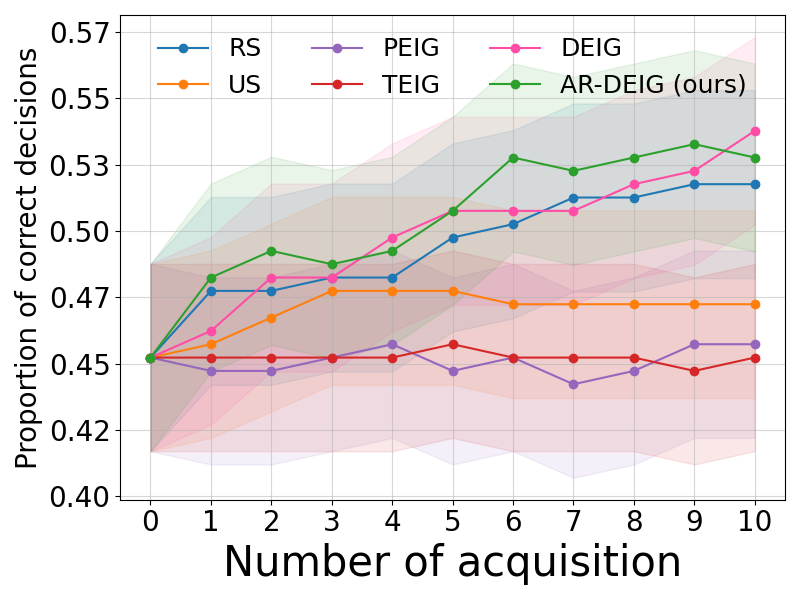}
            \caption*{$\epsilon=0.5$}
        \end{minipage}
        \hfill
        \begin{minipage}{0.48\linewidth}
            \includegraphics[width=\linewidth]{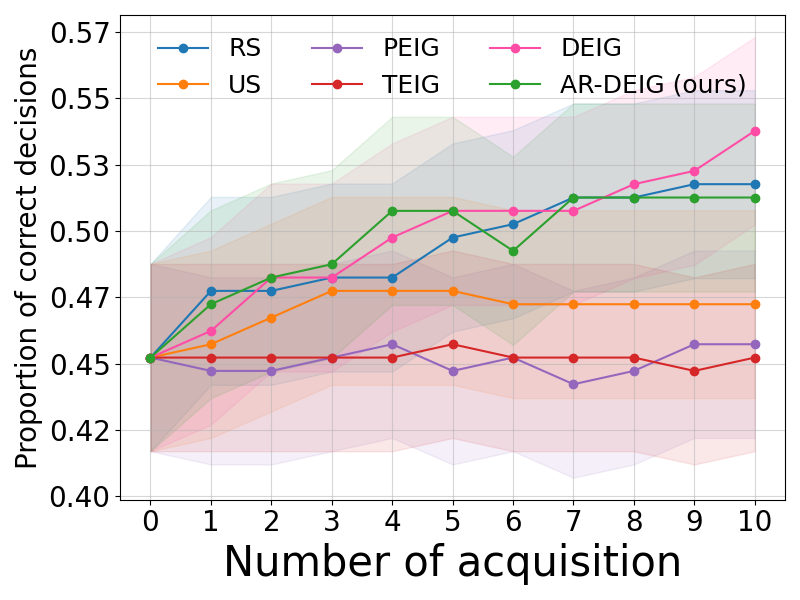}
            \caption*{$\epsilon=0.7$}
        \end{minipage}
        \caption{Accuracy}
        \label{fig:acc_group_supploai}
    \end{subfigure}
    \hfill
    \begin{subfigure}{0.48\linewidth}
        \centering
        \begin{minipage}{0.48\linewidth}
            \includegraphics[width=\linewidth]{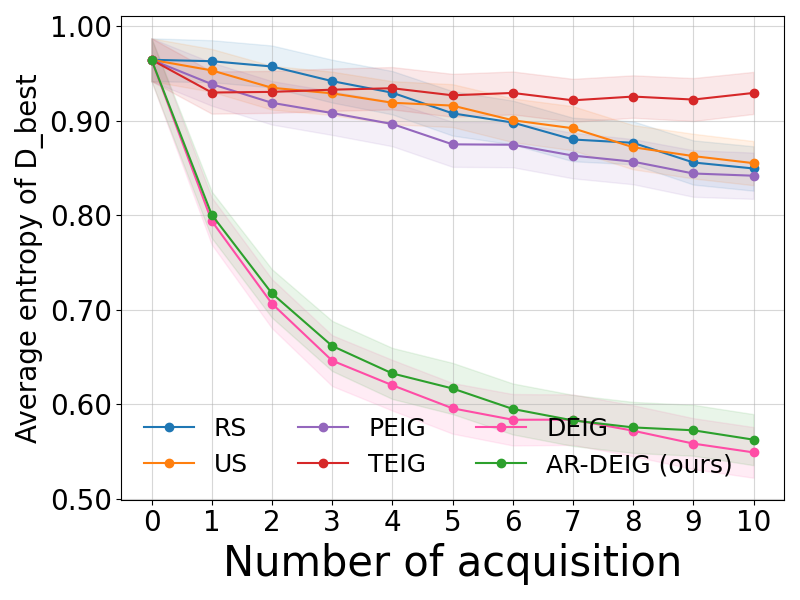}
            \caption*{$\epsilon=0.5$}
        \end{minipage}
        \hfill
        \begin{minipage}{0.48\linewidth}
            \includegraphics[width=\linewidth]{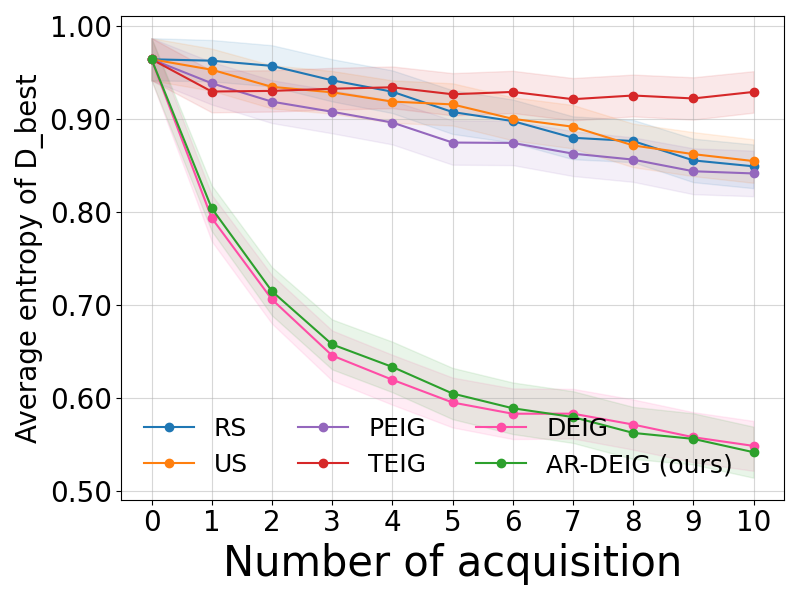}
            \caption*{$\epsilon=0.7$}
        \end{minipage}
        \caption{Entropy}
        \label{fig:ent_group_supploai}
    \end{subfigure}
    \caption{Nominal evaluation on the OAI dataset: (a) accuracy and (b) entropy of the posterior over decisions versus number of acquisitions (200 test points). AR-DEIG ($\epsilon=0.5, 0.7$) remains competitive with baselines, demonstrating that robustness-aware acquisition does not degrade real-world performance.}
    \label{fig:oai_org_su}
\end{figure*}

\subsection{Additional details on Knee Osteoarthritis Dataset and analysis}\label{realdata}
We utilize data from the Osteoarthritis Initiative (OAI)\footnote{\url{https://nda.nih.gov/oai/}}, a large-scale longitudinal study of knee osteoarthritis. We restrict our analysis to symptomatic individuals, defined as subjects with a total WOMAC score greater than 9. From this population, we further select subjects exhibiting early-stage or doubtful radiographic osteoarthritis at baseline, according to the Kellgren--Lawrence (KL) grading system (typically KL grades 1--2).

Disease progression is quantified using changes in joint space width (JSW), a widely used imaging biomarker for cartilage loss. Following established clinical thresholds, a decrease in JSW greater than $0.7\,\mathrm{mm}$ over follow-up is considered indicative of structural progression. JSW measurements are obtained from standardized knee radiographs at a fixed location ($x = 0.25$), thereby focusing on the medial compartment, where osteoarthritis progression is most commonly observed.

All variables are defined at the knee level. The feature set extracted from the OAI dataset includes:
\begin{itemize}
    \item Age at baseline;
    \item Sex;
    \item Body mass index (BMI);
    \item Total WOMAC score;
    \item Knee alignment (varus, valgus, or neutral);
    \item History of knee injury (binary indicator);
    \item History of knee surgery (binary indicator);
    \item Kellgren--Lawrence (KL) grade;
    \item Joint space width (JSW) at $x = 0.25$.
\end{itemize}

These variables capture a combination of demographic, clinical, and radiographic factors commonly associated with osteoarthritis onset and progression.

\subsubsection{Experiment details about Knee Osteoarthritis Dataset}
We model body mass index (BMI) as an adversarial variable (total 9 variables) and apply AR-DEIG with perturbation magnitude $\epsilon = 0.5, 0.7, 20.0$. \cref{fig:oai_org_su} shows the accuracy and entropy of 200 test data and there is no significant difference in performance between the baseline DEIG and AR-DEIG. In terms of accuracy, AR-DEIG slightly outperforms DEIG in later stages, although differences remain within overlapping confidence regions. For entropy, both methods exhibit similar rates of uncertainty reduction, indicating comparable efficiency in information acquisition. These results demonstrate that incorporating adversarial robustness does not degrade performance, while improving stability under non-smooth decision landscapes typical of real-world deployment. \cref{fig:oaiworstcase_20.0} showcase the performance when we use a large $\epsilon$ value of $20.0$.


\begin{figure*}
    \centering
    \includegraphics[width=0.95\linewidth]
 {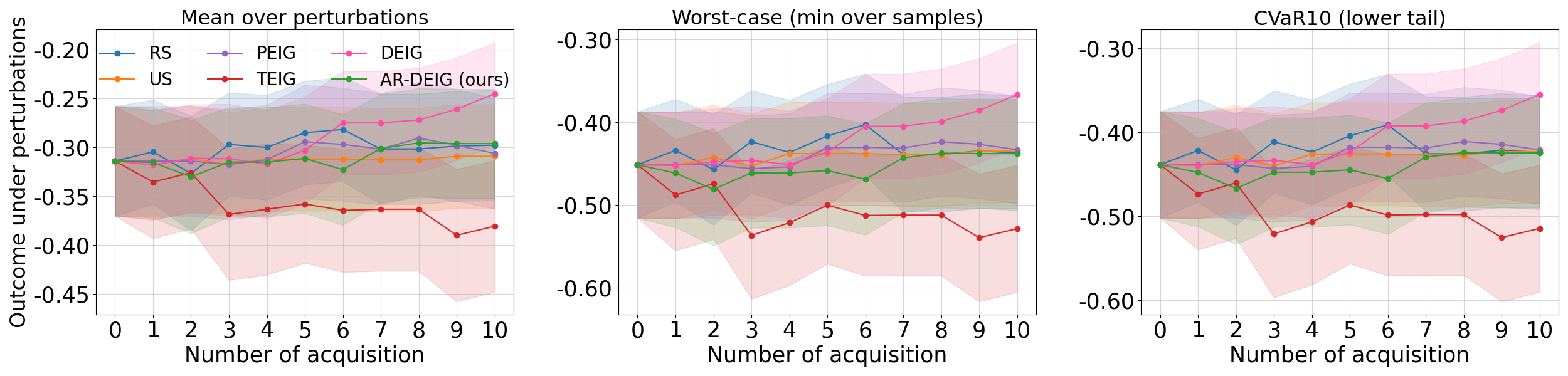}
    \caption{Adversarial Robustness Evaluation: At the higher perturbation budget $\epsilon = 20.0$, AR-DEIG demonstrates degraded robustness across acquisition steps, underperforming relative to several baselines in both average and tail-risk metrics. This is reflected by the green curves attaining lower values in the mean panel and consistently worse outcomes in the worst-case and CVaR10 panels. Results are computed over 500 adversarial samples for 200 test data. Panels show mean (left), worst-case (middle), and CVaR10 (right).}
    \label{fig:oaiworstcase_20.0}
\end{figure*}

\subsubsection{Outcome Model for Decision-Dependent Outcomes}

The linear outcome model is used to estimate the conditional outcome
\( \mathbb{E}[Y \mid X, d] \) and is evaluated for all actions by explicitly
enumerating \( d \in \{1, \dots, K\} \). This enables computation of
decision-specific outcomes without relying on a decision model.

Given a set of evaluation points \( x \), we compute predicted outcomes for \emph{all} possible actions \( d \in \{1, \dots, K\} \) using the outcome model. This yields a matrix of potential outcomes, enabling evaluation of both observed and counterfactual decisions.

To assess robustness, we generate perturbed inputs by adding bounded noise to selected covariate dimensions. For each perturbed point, we evaluate predicted outcomes across all actions using the outcome model. This provides a controlled way to study how small input perturbations affect decision-dependent outcomes.

\paragraph{Linear Outcome Model}
We model the conditional mean outcome as a linear function of covariates, actions, and their interactions. Specifically, for each observation with covariates \( x \in \mathbb{R}^p \) and action \( d \in \{1, \dots, K\} \), we construct a feature vector that includes:
(i) the covariates \( x \), 
(ii) a one-hot encoding of the action \( d \), and 
(iii) interaction terms between \( x \) and each action indicator. Let \( D_k = \mathbf{1}\{d = k\} \) denote the action indicators.

This yields the following functional form:
\begin{equation}
\mathbb{E}[Y \mid X = x, d = k] 
= \beta^\top x + \gamma_k + x^\top \theta_k,
\end{equation}
where \( \beta \in \mathbb{R}^p \) captures effects shared across actions, \( \gamma_k \) is the coefficient associated with action \( k \) (i.e., the effect of \( D_k \)), and \( \theta_k \in \mathbb{R}^p \) captures action-specific heterogeneous effects.

Let \( Z \) denote the resulting design matrix constructed from these components. Model parameters are estimated via ridge-regularized least squares:
\begin{equation}
\hat{w} = (Z^\top Z + \lambda I)^{-1} Z^\top y.
\end{equation}

This specification allows the model to share statistical strength across actions while retaining flexibility to model action-dependent outcome responses. At prediction time, the model can be evaluated for any action \( k \) by constructing the corresponding feature representation, enabling computation of outcomes for all possible decisions.

\subsection{Visualization of 1d Dataset Queries in Active Learning}\label{acq_steps}
The figures illustrate the acquisition process and selected query points for a representative test instance, comparing the baseline \texttt{DEIG} method with our proposed approach, \texttt{AR-DEIG}.

Red stars denote the query points selected during each acquisition step and the black vertical line is the point decision-maker has to make a decision. The dataset comprises three decisions, with outcomes color-coded as follows: blue for decision 1, orange for decision 2, and green for decision 3. In real-world applications, these decisions correspond to alternative treatment plans.

The bar chart at the top of each plot shows the posterior probability over decisions after incorporating the selected query point and updating the model.
\subsubsection{DEIG based Querying over Acquisition steps}
\cref{fig:acq_deig_1}--\cref{fig:acq_deig_10} correspond to the decision-based EIG approach. Notably, this method exhibits over-exploitation, repeatedly selecting query points near local optima while failing to adequately explore surrounding brittle regions.
\subsubsection{AR-DEIG based Querying over Acquisition steps}
\cref{fig:acq_ardeig_1}--\cref{fig:acq_ardeig_10} correspond to the decision-based EIG approach.
Our method exhibits a more exploratory acquisition behavior, which facilitates improved identification of robust optima over successive iterations.
\begin{figure*}
    \centering
    \includegraphics[width=0.95\linewidth]{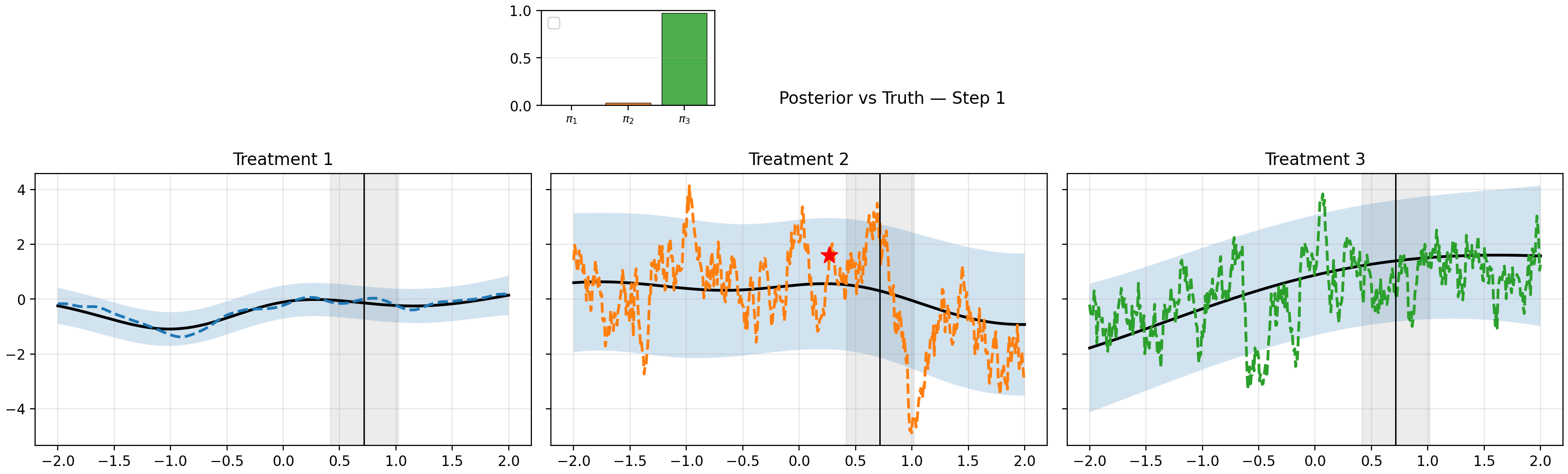}
    \caption{Decision-EIG: Acquisition step 1 (red star indicates the query point and black vertical line is the new point the decision-maker has to make a decision).}
    \label{fig:acq_deig_1}
\end{figure*}
\begin{figure*}
    \centering
    \includegraphics[width=0.95\linewidth]{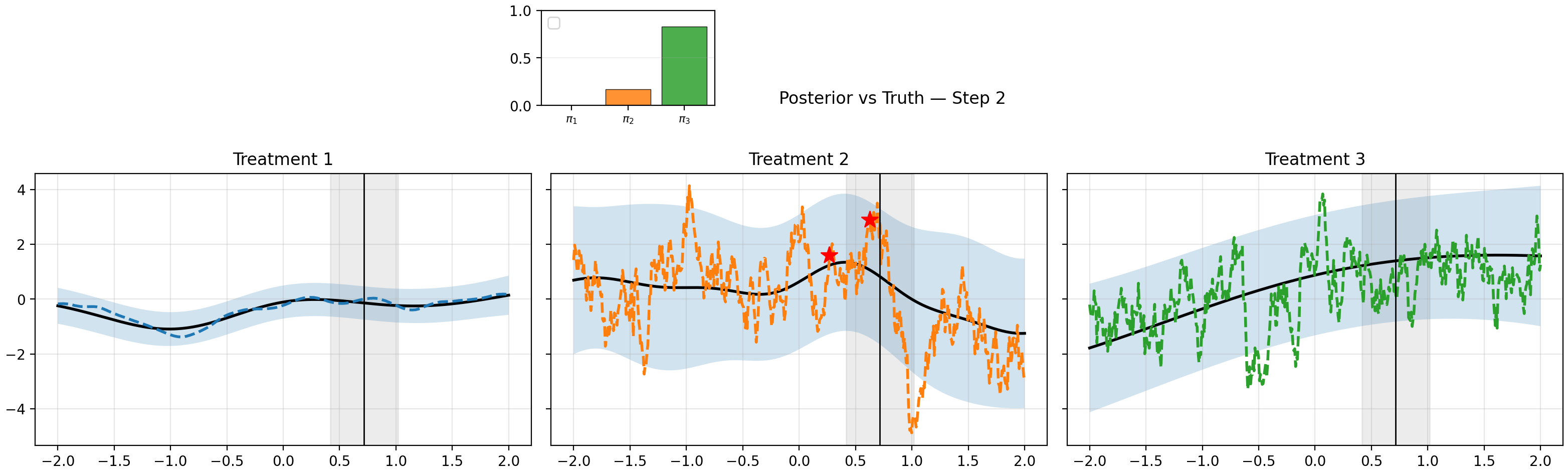}
    \caption{Decision-EIG: Acquisition step 2 (red star indicates the query point and black vertical line is the new point the decision-maker has to make a decision).}
    \label{fig:acq_deig_2}
\end{figure*}
\begin{figure*}
    \centering
    \includegraphics[width=0.95\linewidth]{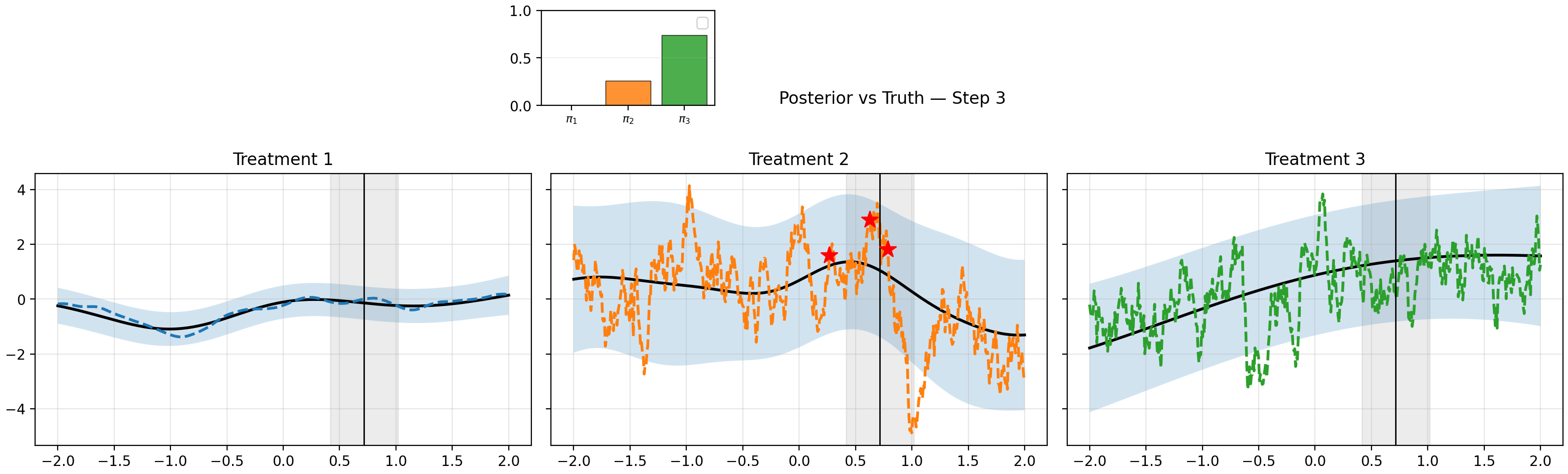}
    \caption{Decision-EIG: Acquisition step 3 (red star indicates the query point and black vertical line is the new point the decision-maker has to make a decision).}
    \label{fig:acq_deig_3}
\end{figure*}
\begin{figure*}
    \centering
    \includegraphics[width=0.95\linewidth]{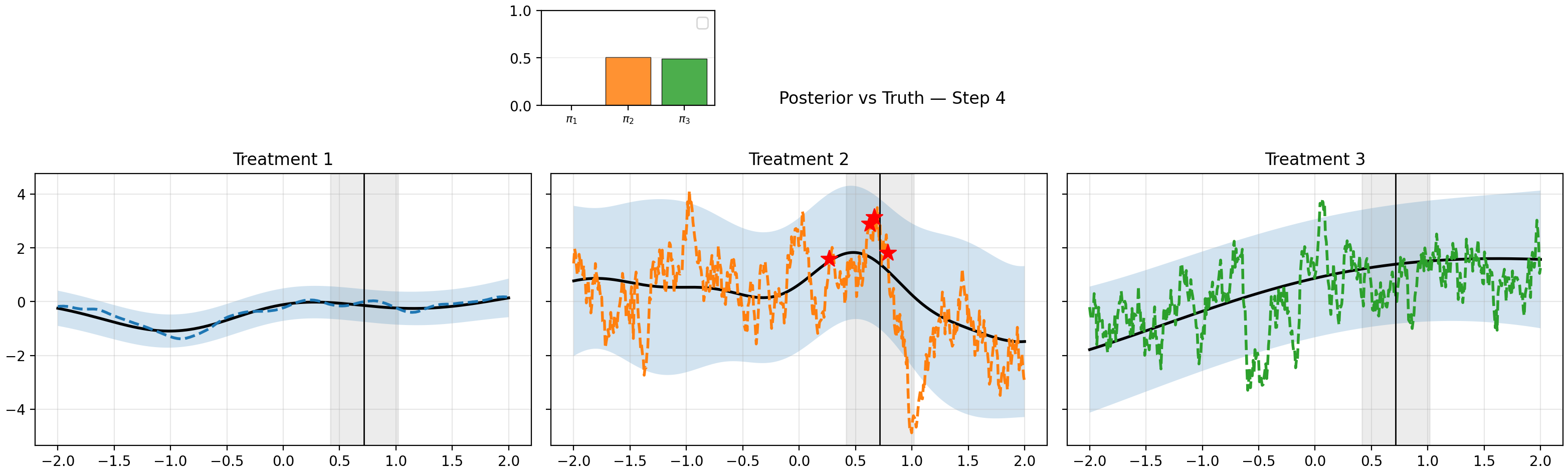}
    \caption{Decision-EIG: Acquisition step 4 (red star indicates the query point and black vertical line is the new point the decision-maker has to make a decision).}
    \label{fig:acq_deig_4}
\end{figure*}
\begin{figure*}
    \centering
    \includegraphics[width=0.95\linewidth]{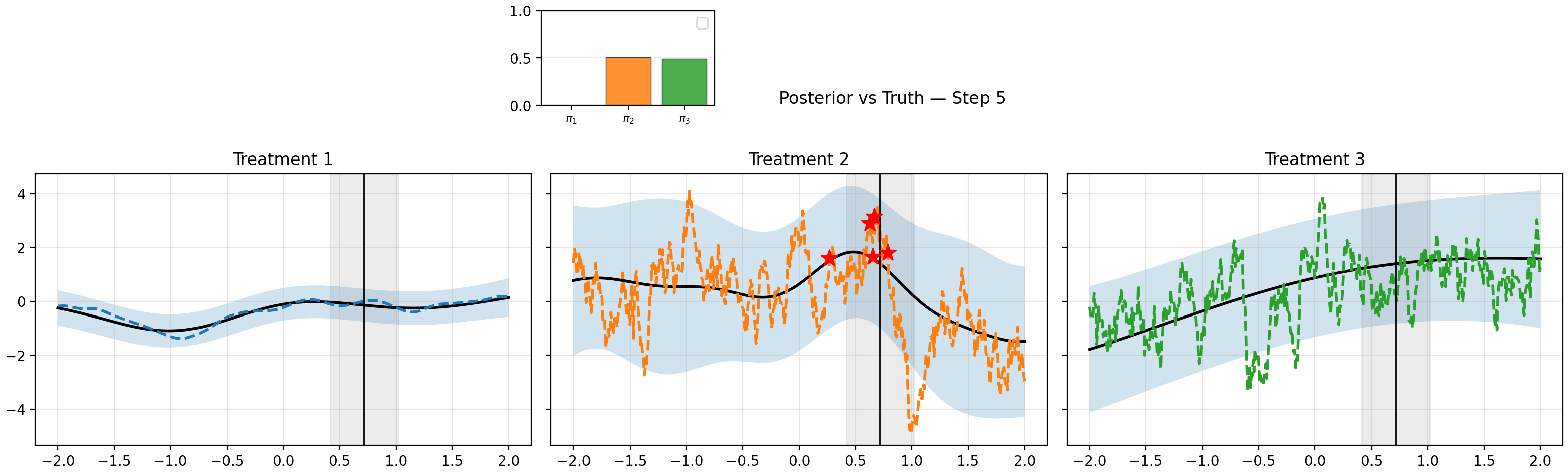}
    \caption{Decision-EIG: Acquisition step 5 (red star indicates the query point and black vertical line is the new point the decision-maker has to make a decision).}
    \label{fig:acq_deig_5}
\end{figure*}
\begin{figure*}
    \centering
    \includegraphics[width=0.95\linewidth]{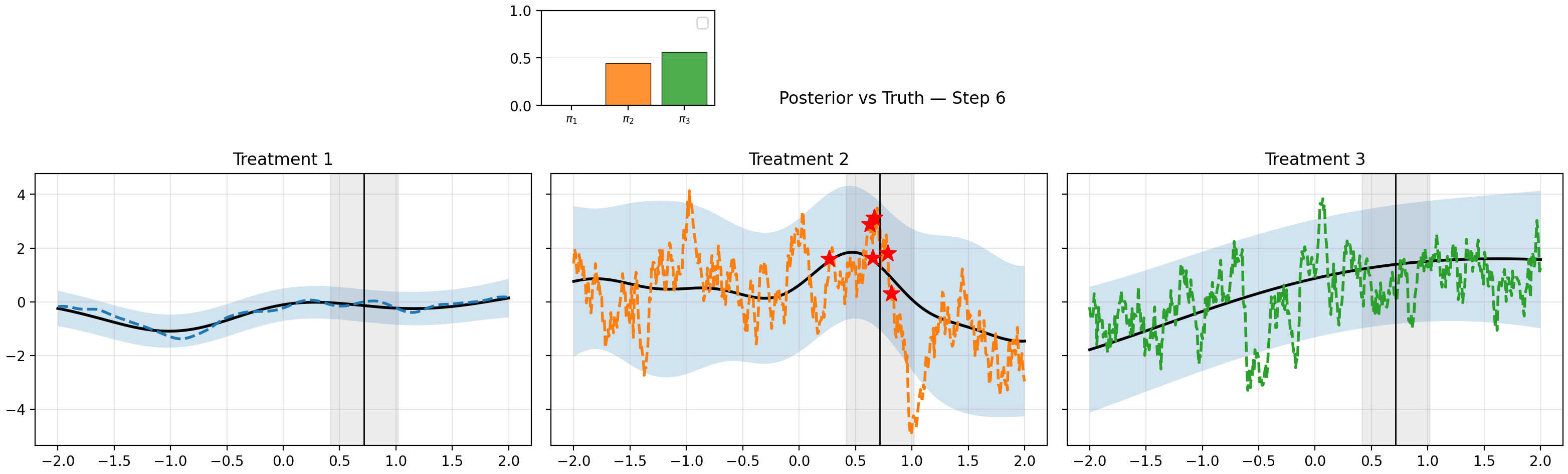}
    \caption{Decision-EIG: Acquisition step 6 (red star indicates the query point and black vertical line is the new point the decision-maker has to make a decision).}
    \label{fig:acq_deig_6}
\end{figure*}
\begin{figure*}
    \centering
    \includegraphics[width=0.95\linewidth]{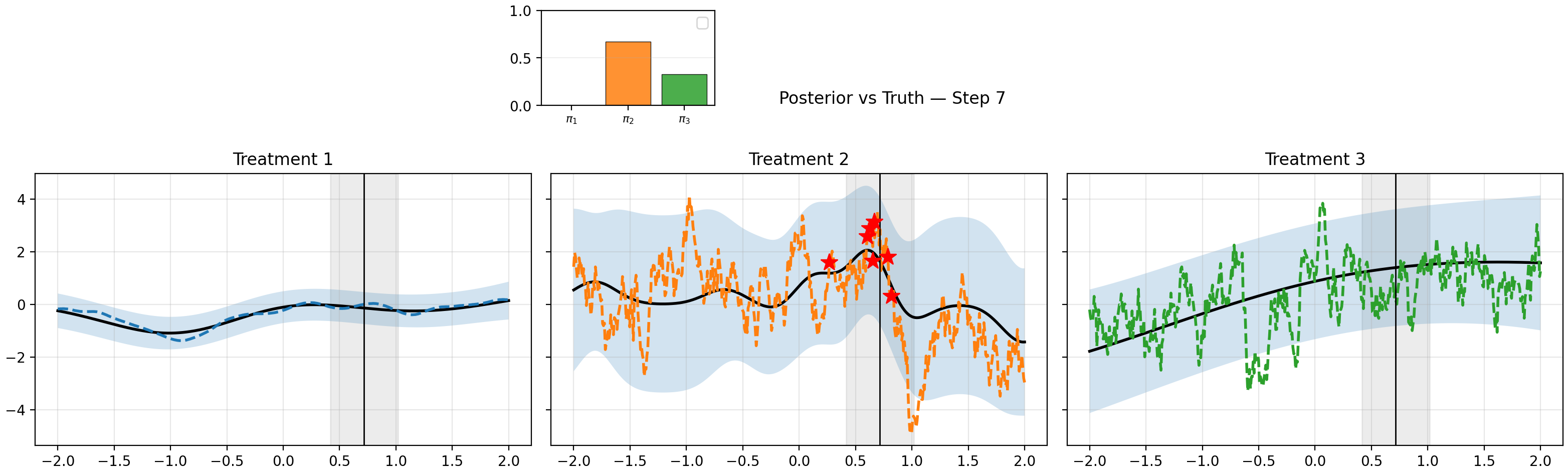}
    \caption{Decision-EIG: Acquisition step 7 (red star indicates the query point and black vertical line is the new point the decision-maker has to make a decision).}
    \label{fig:acq_deig_7}
\end{figure*}
\begin{figure*}
    \centering
    \includegraphics[width=0.95\linewidth]{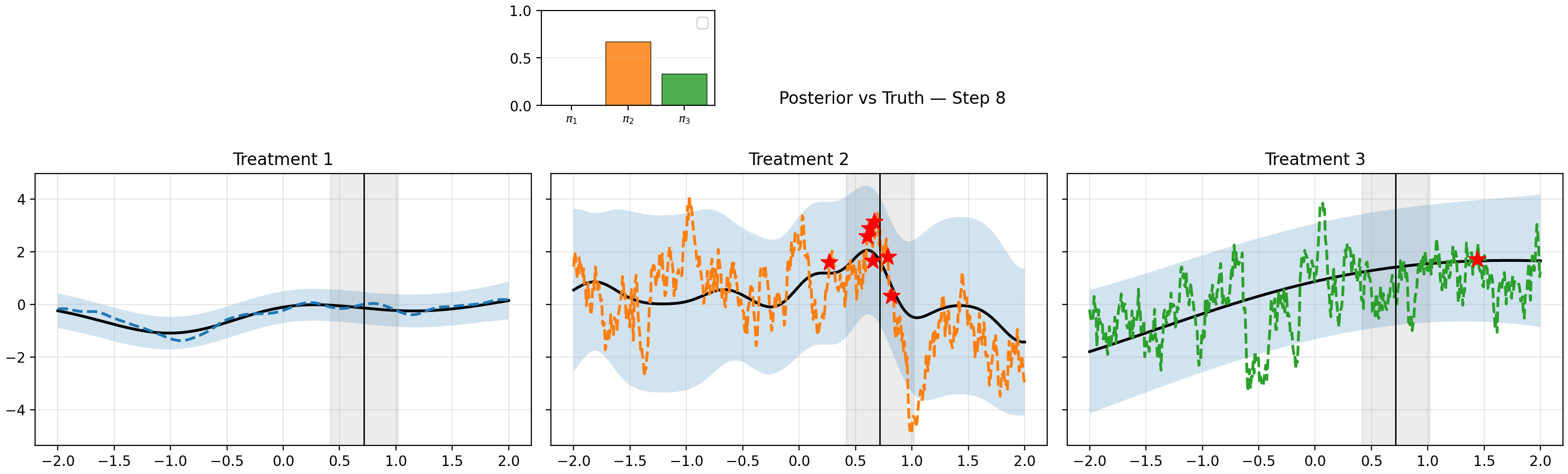}
    \caption{Decision-EIG: Acquisition step 8 (red star indicates the query point and black vertical line is the new point the decision-maker has to make a decision).}
    \label{fig:acq_deig_8}
\end{figure*}
\begin{figure*}
    \centering
    \includegraphics[width=0.95\linewidth]{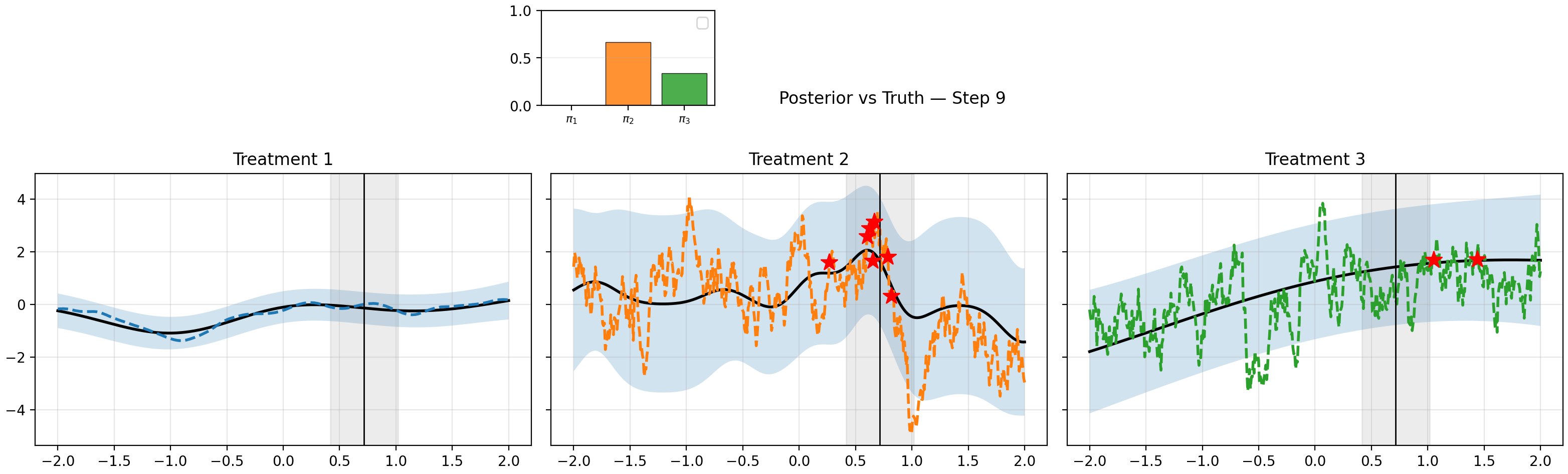}
    \caption{Decision-EIG: Acquisition step 9 (red star indicates the query point and black vertical line is the new point the decision-maker has to make a decision).}
    \label{fig:acq_deig_9}
\end{figure*}
\begin{figure*}
    \centering
    \includegraphics[width=0.95\linewidth]{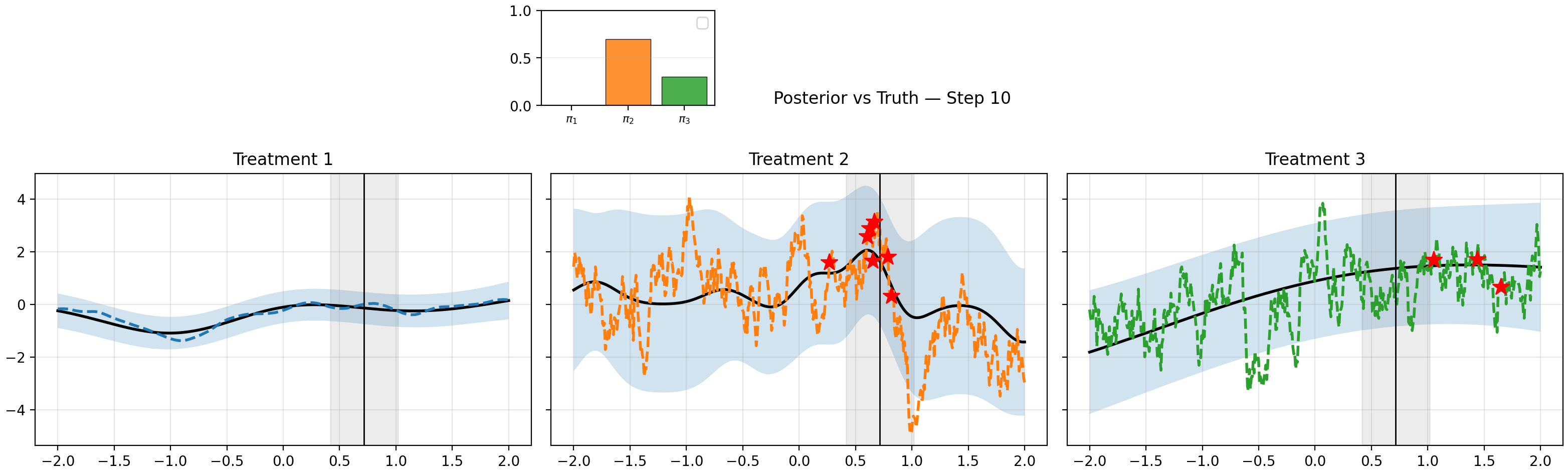}
    \caption{Decision-EIG: Acquisition step 10 (red star indicates the query point and black vertical line is the new point the decision-maker has to make a decision).}
    \label{fig:acq_deig_10}
\end{figure*}
\begin{figure*}
    \centering
    \includegraphics[width=0.95\linewidth]{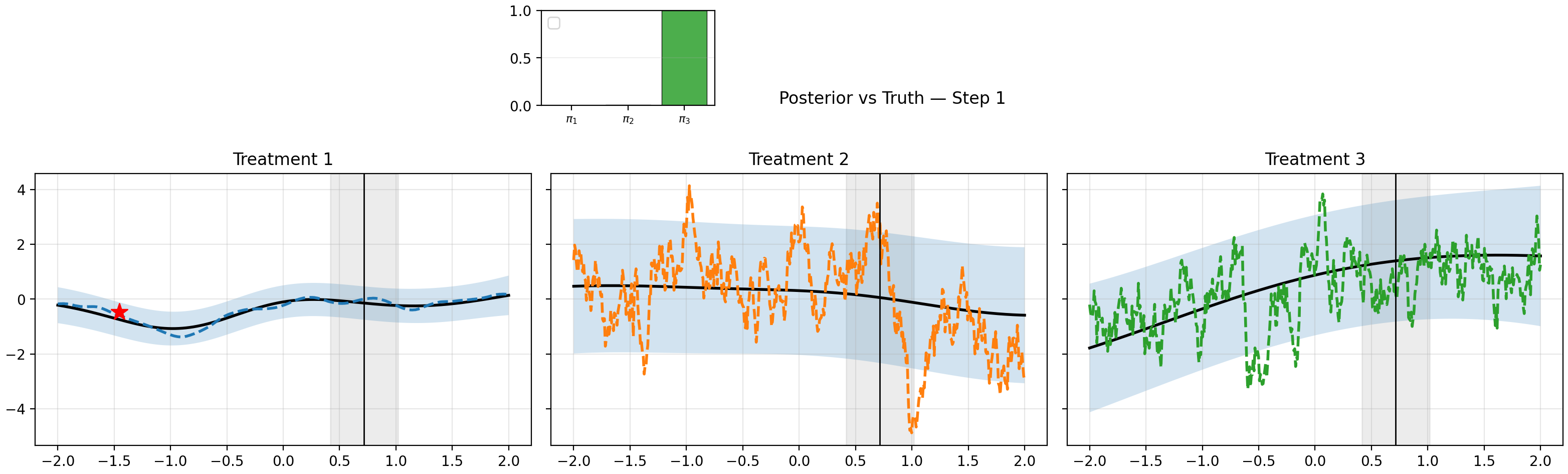}
    \caption{AR-DEIG: Acquisition step 1 (red star indicates the query point and black vertical line is the new point the decision-maker has to make a decision).}
    \label{fig:acq_ardeig_1}
\end{figure*}
\begin{figure*}
    \centering
    \includegraphics[width=0.95\linewidth]{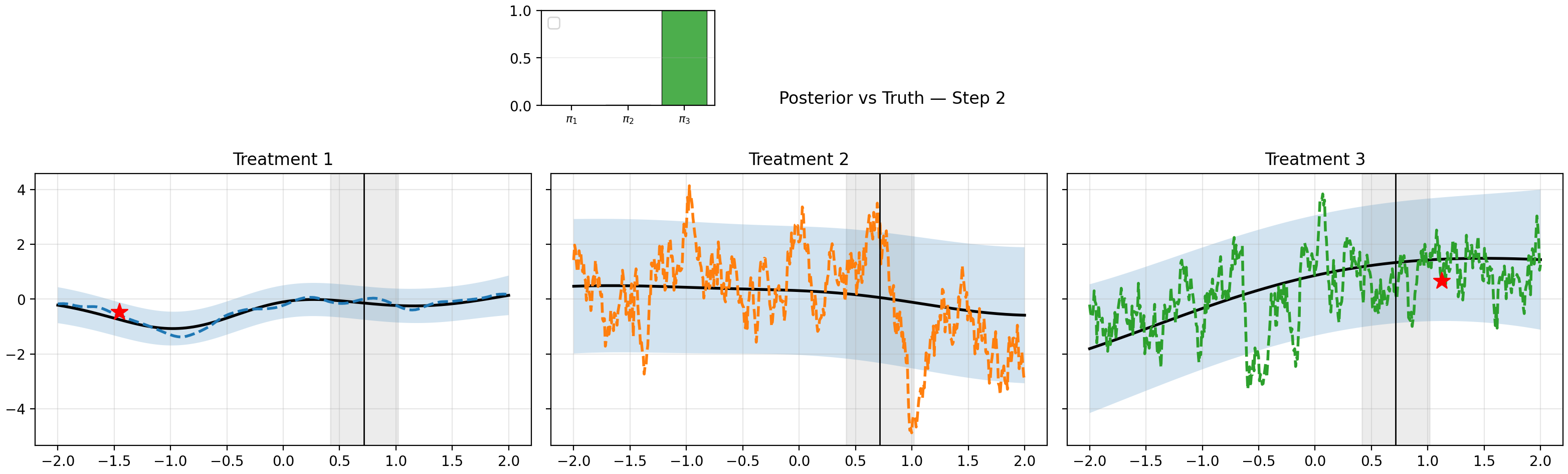}
    \caption{AR-DEIG: Acquisition step 2 (red star indicates the query point and black vertical line is the new point the decision-maker has to make a decision).}
    \label{fig:acq_ardeig_2}
\end{figure*}
\begin{figure*}
    \centering
    \includegraphics[width=0.95\linewidth]{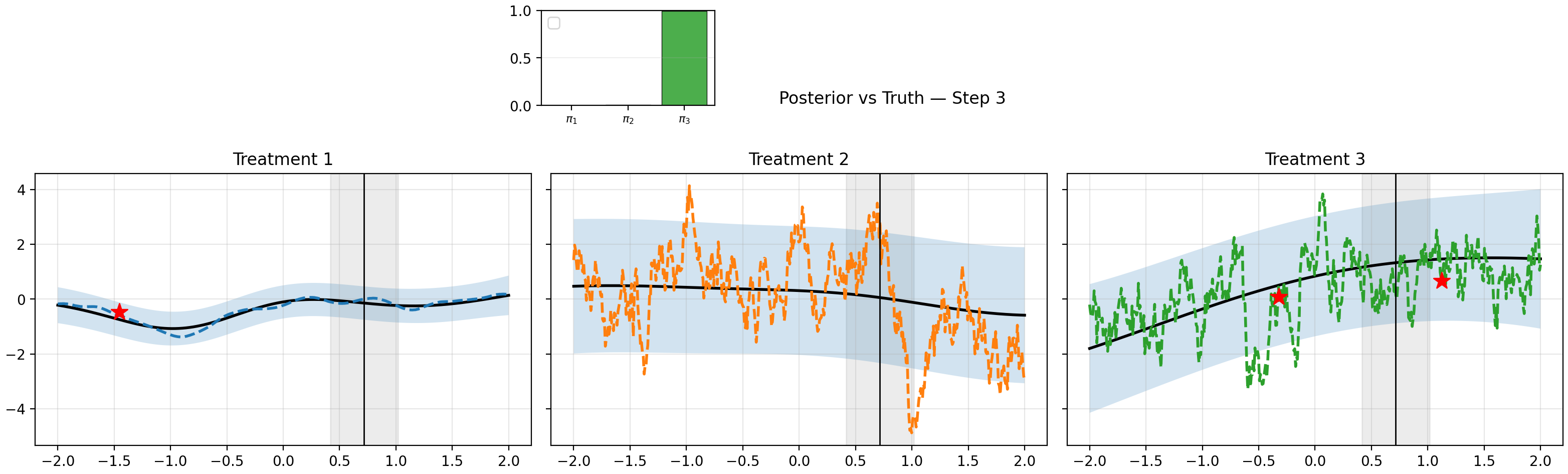}
    \caption{AR-DEIG: Acquisition step 3 (red star indicates the query point and black vertical line is the new point the decision-maker has to make a decision).}
    \label{fig:acq_ardeig_3}
\end{figure*}
\begin{figure*}
    \centering
    \includegraphics[width=0.95\linewidth]{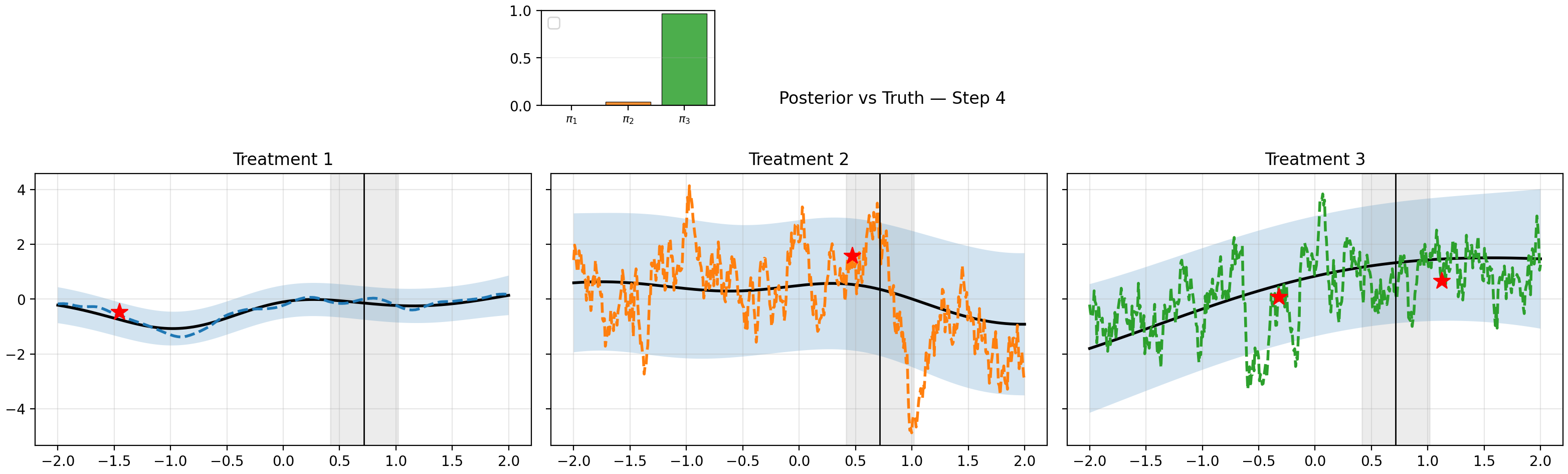}
    \caption{AR-DEIG: Acquisition step 4 (red star indicates the query point and black vertical line is the new point the decision-maker has to make a decision).}
    \label{fig:acq_ardeig_4}
\end{figure*}
\begin{figure*}
    \centering
    \includegraphics[width=0.95\linewidth]{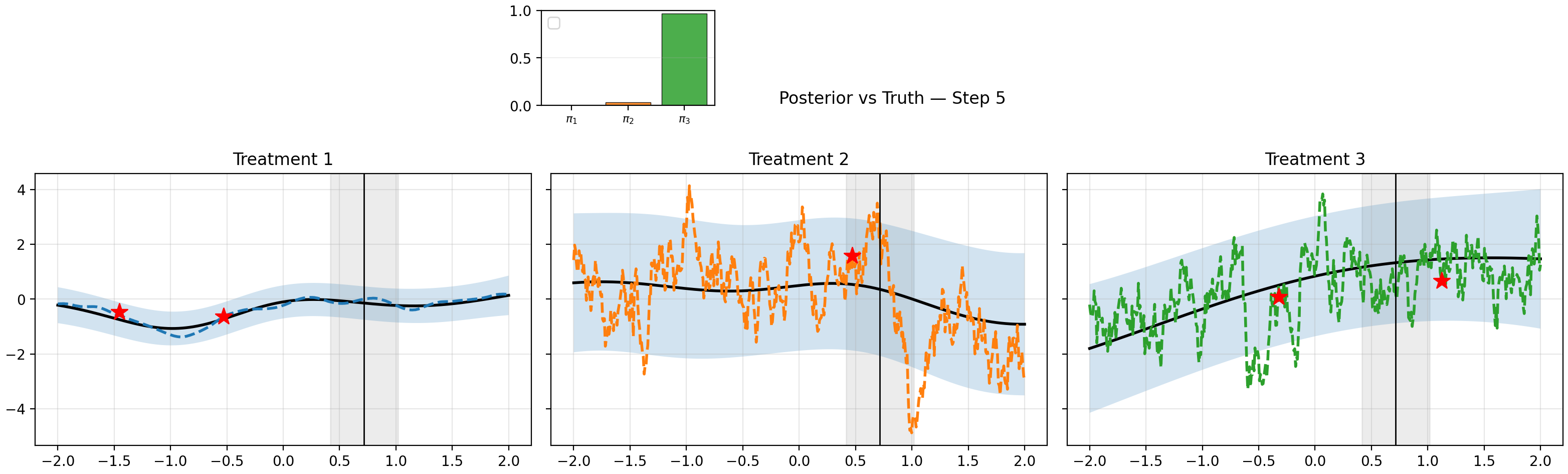}
    \caption{AR-DEIG: Acquisition step 5 (red star indicates the query point and black vertical line is the new point the decision-maker has to make a decision).}
    \label{fig:acq_ardeig_5}
\end{figure*}
\begin{figure*}
    \centering
    \includegraphics[width=0.95\linewidth]{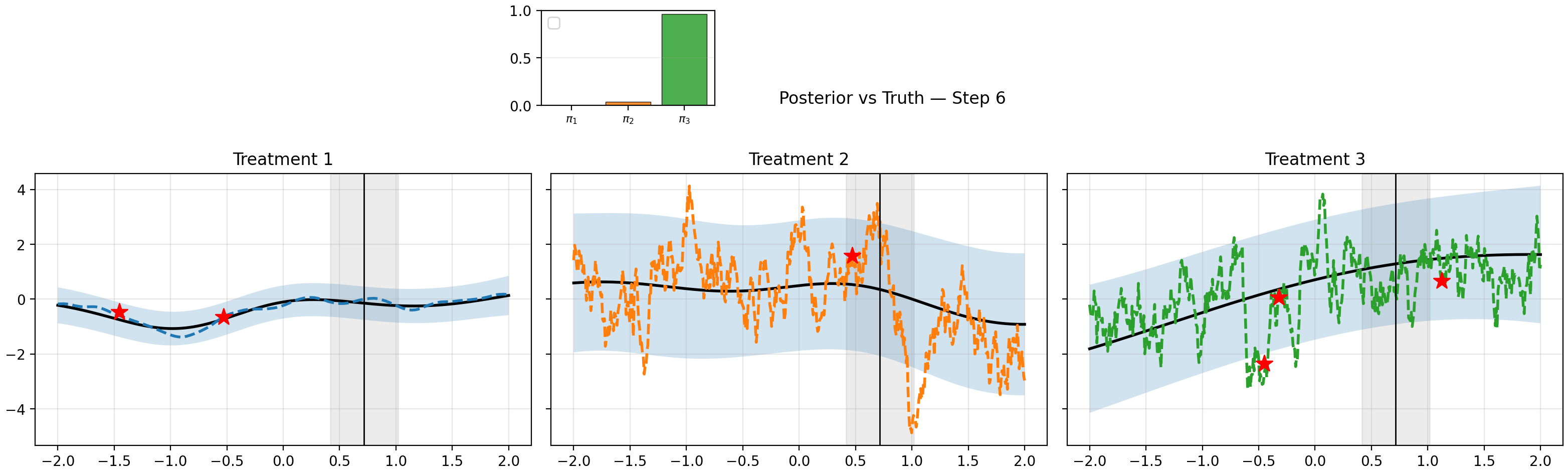}
    \caption{AR-DEIG: Acquisition step 6 (red star indicates the query point and black vertical line is the new point the decision-maker has to make a decision).}
    \label{fig:acq_ardeig_6}
\end{figure*}
\begin{figure*}
    \centering
    \includegraphics[width=0.95\linewidth]{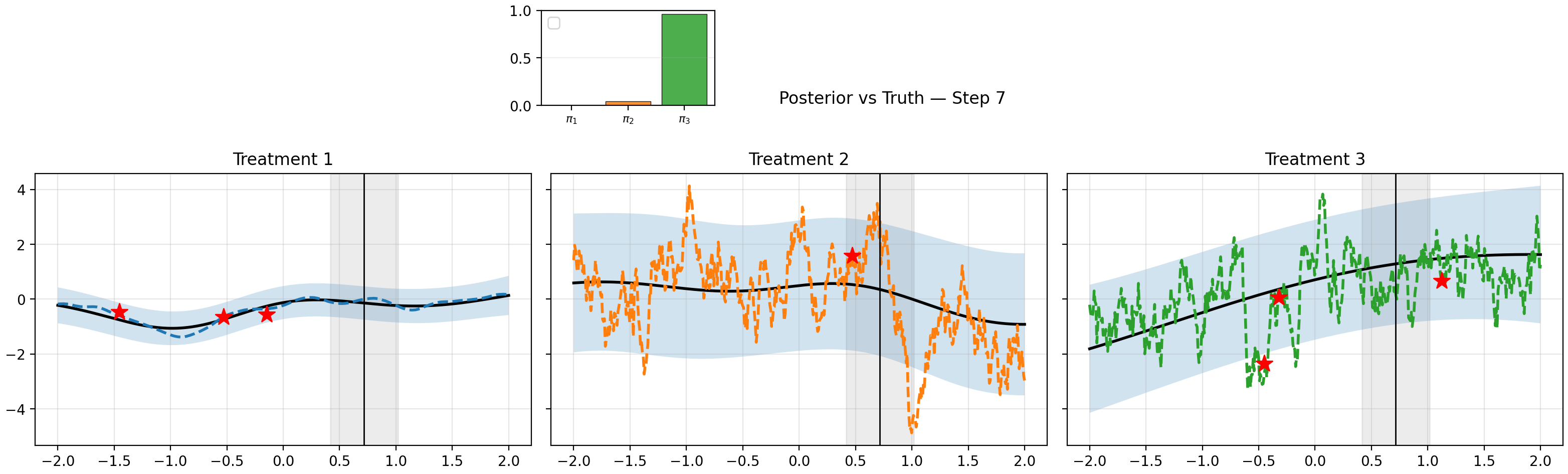}
    \caption{AR-DEIG: Acquisition step 7 (red star indicates the query point and black vertical line is the new point the decision-maker has to make a decision).}
    \label{fig:acq_ardeig_7}
\end{figure*}
\begin{figure*}
    \centering
    \includegraphics[width=0.95\linewidth]{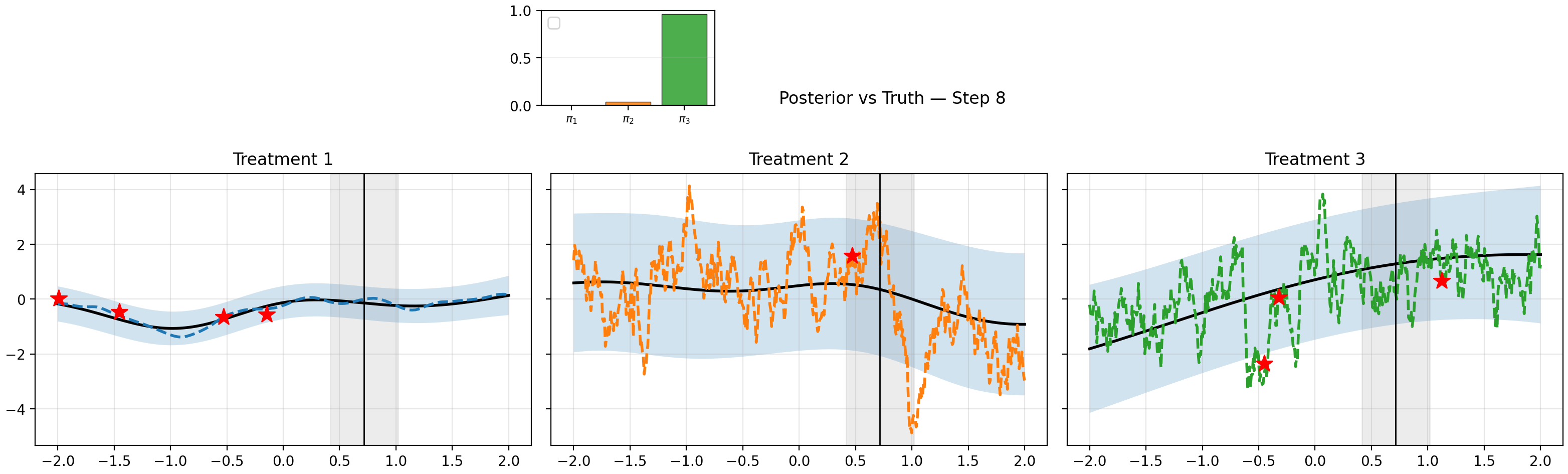}
    \caption{AR-DEIG: Acquisition step 8 (red star indicates the query point and black vertical line is the new point the decision-maker has to make a decision).}
    \label{fig:acq_ardeig_8}
\end{figure*}
\begin{figure*}
    \centering
    \includegraphics[width=0.95\linewidth]{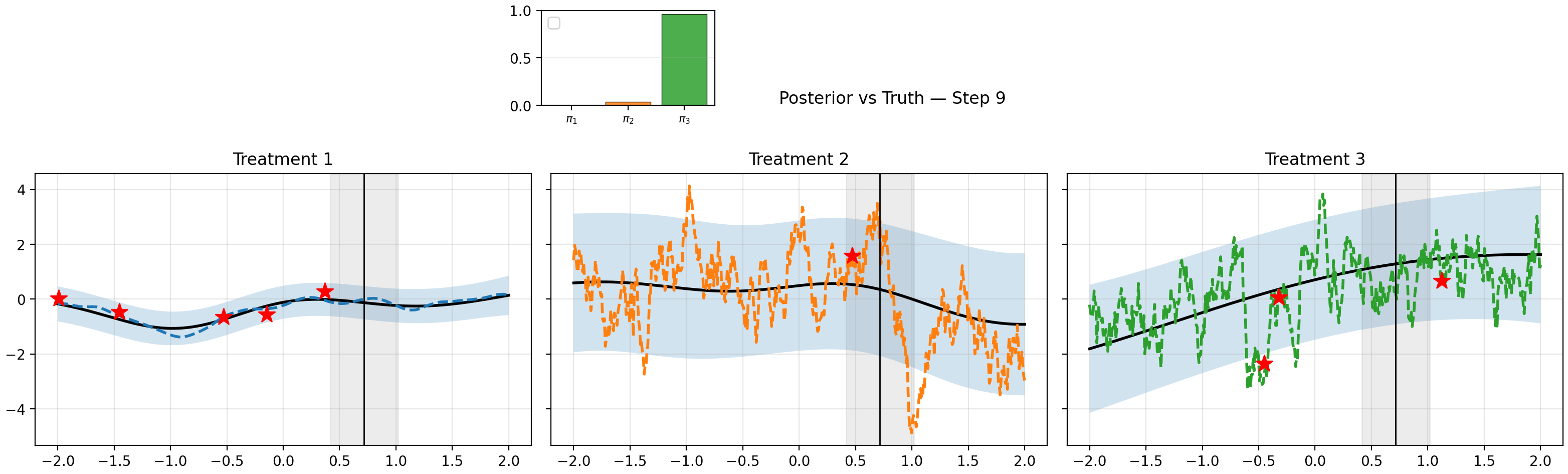}
    \caption{AR-DEIG: Acquisition step 9 (red star indicates the query point and black vertical line is the new point the decision-maker has to make a decision).}
    \label{fig:acq_ardeig_9}
\end{figure*}
\begin{figure*}
    \centering
    \includegraphics[width=0.95\linewidth]{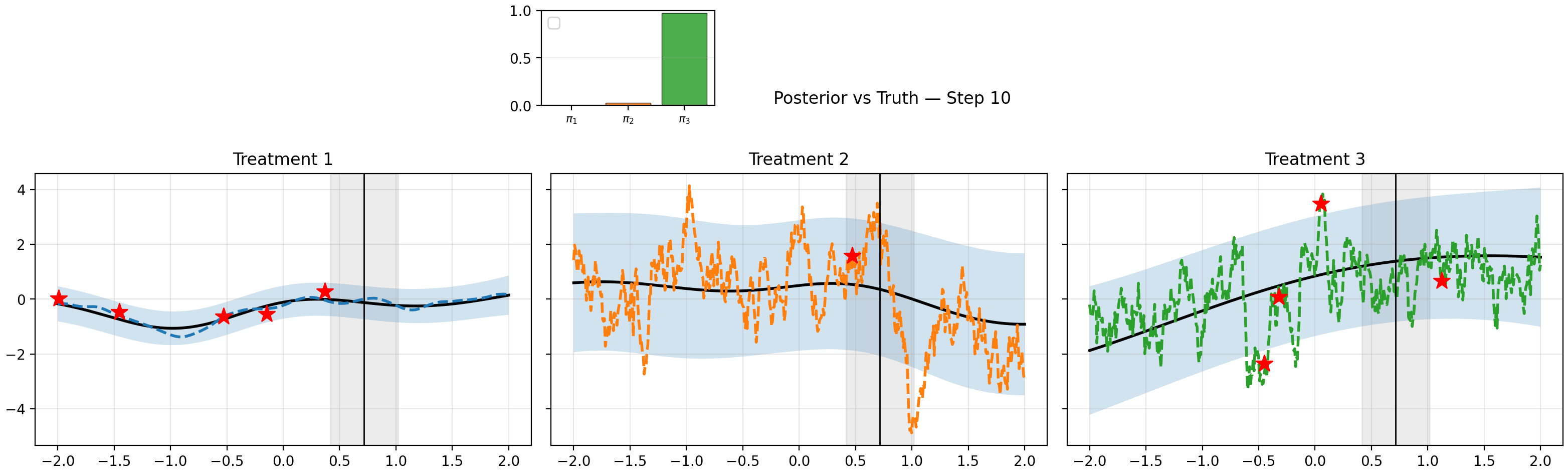}
    \caption{AR-DEIG: Acquisition step 10 (red star indicates the query point and black vertical line is the new point the decision-maker has to make a decision).}
    \label{fig:acq_ardeig_10}
\end{figure*}
\end{document}